%% file: paper.tex
\documentclass{article}

\newcommand{\ourmaintitle}{Root Cause Analysis of\\Measurement and Mechanistic Anomalies}

\newcommand{\ourtitle}{\ourmaintitle}
\newcommand{\ourmethod}{\textsc{Cali}\xspace}
\newcommand{\ourmethodstar}{\textsc{Cali}*\xspace}
\newcommand{\ourmethodprime}{\textsc{Cali} \xspace}

\RequirePackage{amsmath}
\RequirePackage{amssymb}
\PassOptionsToPackage{numbers,sort&compress}{natbib}
\input{preamble}

\RequirePackage{mathtools}
\mathtoolsset{showonlyrefs}
\usepackage[final,main]{neurips_2026} 
\begin{document}
\setlength{\pdfpagewidth}{8.5in}
\setlength{\pdfpageheight}{11in}

\title{\ourtitle}
\author{Hendrik Suhr, David Kaltenpoth and Jilles Vreeken\\
	CISPA Helmholtz Center for Information Security\\ 
	\texttt{\{hendrik.suhr,david.kaltenpoth,jv\}@cispa.de}
}

	\maketitle

\begin{abstract}
\input{abstract}
\end{abstract}

\input{introduction}

\input{related_work}

\input{preliminaries}

\input{theory2.0}

\input{instantiation}

\input{experiments}

\input{conclusion}

\begin{ack}
	H.S. is a PhD student at the Saarbrücken Graduate School of Computer Science and was supported by the Liberate AI project, funded by the Initiative and Networking Fund of the Helmholtz Association.
\end{ack}
\clearpage
\bibliography{bib/abbreviations,bib/bib-paper}
\bibliographystyle{abbrvnat}

\clearpage 
%% apx
\ifapx
\appendix
\onecolumn
\include{appendix}
\fi
\input{checklist}

\end{document}

%% file: preamble.tex
\newif\ifapx
\apxtrue % uncomment if you want appendix
\newif\ifpaper 
\paperfalse
\usepackage[dvipsnames]{xcolor}
\usepackage{xspace}
\usepackage{diagbox} % put this in your preamble
\usepackage{bigints}
\usepackage{wrapfig}
\usepackage{algorithm}
\usepackage{algorithmic}
\usepackage{amsthm}
\usepackage{latexsym}
\usepackage{subcaption}
\usepackage{graphicx,centernot}

\usepackage[notheorems]{eda} % The EDA package, includes many default packages

\usepackage{footmisc}
\usepackage{pgfplotstable}
\usepackage{pgfplotstable}

\usepackage{dblfloatfix}
\usepackage{afterpage}
\usepackage{pgfplots}
\usepackage{pgfplotstable}
\usepackage{booktabs}
\usetikzlibrary{patterns}
\usepackage{tabularx, booktabs}
\usepackage{dsfont}
\pgfplotsset{compat=1.18}
\usepgfplotslibrary{groupplots}

\usetikzlibrary{arrows.meta, positioning,shapes.geometric, arrows}

\input{defines}
\usetikzlibrary{pgfplots.fillbetween}
\usepackage{amssymb}  % <- provides \lightning
\usepackage{fontawesome5}

\usetikzlibrary{patterns}

\usepackage[normalem]{ulem}

%% file: defines.tex
\definecolor{OIblue}{HTML}{0072B2}
\definecolor{OIvermillion}{HTML}{D55E00}

\makeatletter
\newcommand\resetstackedplots{
\makeatletter
\pgfplots@stacked@isfirstplottrue
\makeatother
\addplot [forget plot,draw=none] coordinates{AP};
}
\makeatother

\newcommand{\g}{\textcolor{OIvermillion}{\ensuremath{\mathbf g}}\xspace}
\newcommand{\m}{\textcolor{OIblue}{\ensuremath{\mathbf m}}\xspace}
\newcommand{\Xt}{\ensuremath{\mathbf X^\ast}}

\newcommand{\Xtj}{\ensuremath{X^\ast_j}}

\newcommand{\Xb}{\ensuremath{\mathbf{X}}}

\newcommand{\xb}{\ensuremath{\mathbf{x}}}
\newcommand{\Gc}{{\ensuremath{\mathcal{G}}}}

\newcommand{\Gci}{{\ensuremath{\mathcal{G}^{(i)}}}}
\newcommand{\xbi}{\ensuremath{\mathbf{x}^{(i)}}}
\newcommand{\Vb}{\ensuremath{\mathbf V}}

\newcommand{\Meas}{\ensuremath{W}}
\newcommand{\Gen}{\ensuremath{Z}}
\newcommand{\Measb}{\ensuremath{\mathbf W}}
\newcommand{\Genb}{\ensuremath{\mathbf Z}}

\newcommand{\meas}{\ensuremath{w}}
\newcommand{\gen}{\ensuremath{z}}
\newcommand{\measb}{\ensuremath{\mathbf \meas}}
\newcommand{\genb}{\ensuremath{\mathbf \gen}}

\newcommand{\ass}{\ensuremath{\mathbf a}}

\newcommand{\MeasRV}{\ensuremath{{O^\mathit{meas}}}}
\newcommand{\GenRV}{{\ensuremath{O^\mathit{mech}}}}
\newcommand{\MeasRVj}{\ensuremath{{O_j^\mathit{meas}}}}
\newcommand{\GenRVj}{{\ensuremath{O_j^\mathit{mech}}}}

\newcommand{\MeasPj}{\ensuremath{{\pi_j^\mathit{meas}}}}
\newcommand{\GenPj}{\ensuremath{{\pi_j^\mathit{mech}}}}
\newcommand{\pib}{\ensuremath{\boldsymbol \pi}}
\newcommand{\qb}{\ensuremath{\boldsymbol q}}

\DeclareMathOperator{\Pa}{Pa}
\DeclareMathOperator{\MB}{MB}
\DeclareMathOperator{\pa}{pa}
\newcommand{\Paj}{\Pa_j}

\newcommand{\Mc}{{\ensuremath{\mathcal{\Meas}}}}
\newcommand{\Genc}{{\ensuremath{\mathcal{\Gen}}}}

\newcommand{\MOE}{\textsc{M2oe}\xspace}
\newcommand{\DICE}{\textsc{Dice}\xspace}

\newcommand{\SMOOTHTRAVERSAL}{\textsc{SmTr}\xspace}
\newcommand{\SMOOTHTRAVERSALSTAR}{\textsc{SmTr}*\xspace}
\newcommand{\SCOREORDER}{\textsc{ScOr}\xspace}
\newcommand{\MARGINAL}{\textsc{Marg}\xspace}
\newcommand{\MARGINALSTAR}{\textsc{Marg}*\xspace}

\newcommand{\SHAP}{\textsc{Shap}\xspace}
\newcommand{\LIME}{\textsc{Lime}\xspace}
\newcommand{\NOTEARSMLP}{\textsc{NoTears-mlp}\xspace}
\newcommand{\TOPIC}{\textsc{Topic}\xspace}
\newcommand{\CAM}{\textsc{Cam}\xspace}
\newcommand{\NOGAM}{\textsc{NoGAM}\xspace}
\newcommand{\SCORE}{\textsc{Score}\xspace}
\newcommand{\OPTCLASS}{\textsc{Opt}*\xspace}

\newcommand{\COPOD}{\textsc{COPOD}\xspace}
\newcommand{\IFOREST}{\textsc{iForest}\xspace}
\newcommand{\DEEPSVDD}{\textsc{DEEPSVDD}\xspace}
\newcommand{\LOF}{\textsc{LOF}\xspace}
\newcommand{\NORMALFLOW}{\textsc{NF}\xspace}

\newcommand{\AUTOENCODER}{\textsc{AE}\xspace}
\newcommand{\AESHAP}{\textsc{AE-Shap}\xspace}
\newcommand{\AEREC}{\textsc{AE-REC}\xspace}

\newcommand{\SCM}{\textsc{SCM}*\xspace}

\theoremstyle{plain}
\newtheorem{theorem}{Theorem}[section]
\newtheorem*{theorem*}{Theorem} % for unnumbered restatement

\newtheorem{lemma}[theorem]{Lemma}

\theoremstyle{definition}
\newtheorem{definition}[theorem]{Definition}

\theoremstyle{remark}
\newtheorem{remark}[theorem]{Remark}
\theoremstyle{definition}

\definecolor{OurColor}{RGB}{0,180,0}        % bright green (keep)
\definecolor{OurCAMColor}{RGB}{0,100,0}     % dark green (keep)
\definecolor{OurColorCLIP}{RGB}{0,0,180}        % bright green (keep)
\definecolor{OurCAMColorCLIP}{RGB}{0,000,100}     % dark green (keep)
\definecolor{AEColor}{RGB}{230,25,75}       % vivid raspberry red
\definecolor{NFColor}{RGB}{245,130,48}      % bright tangerine
\definecolor{LOFColor}{RGB}{255,225,25}     % sunshine yellow
\definecolor{IsoColor}{RGB}{60,180,220}     % sky blue
\definecolor{DeepSVDDColor}{RGB}{145,30,180} % purple
\definecolor{CopodColor}{RGB}{255,190,200}  % soft pink

\definecolor{SOColor}{RGB}{240,50,230}      % magenta
\definecolor{DICEColor}{RGB}{0,130,200}     % teal blue
\definecolor{MOEColor}{RGB}{230,190,255}    % lavender
\definecolor{STColor}{RGB}{245,130,48}     % sunshine yellow
\definecolor{SCMColor}{RGB}{255,225,25}      % hot pink

\definecolor{MarginalColor}{RGB}{245,130,48}     % bright tangerine
\definecolor{MarginalLearnColor}{RGB}{230,25,75} % vivid raspberry red

\definecolor{CAMColor}{RGB}{145,30,180}      % purple
\definecolor{SCOREColor}{RGB}{60,180,220}    % sky blue
\definecolor{NOTEARSColor}{RGB}{255,105,180} % hot pink
\definecolor{NOGAMColor}{RGB}{255,190,200}   % soft pink
\definecolor{TOPICColor}{RGB}{255,225,25}    % sunshine yellow

\DeclareMathOperator*{\argmax}{arg\, max}

%% file: abstract.tex
Root cause analysis of anomalies aims to identify how and why a sample deviates from the normal process. 
Existing methods primarily focus on telling which features are responsible, ignoring that anomalies can arise through two fundamentally different processes: measurement errors, where the sample is generated normally but one or more values is recorded incorrectly, and mechanism shifts, where the causal process that generated the sample was changed. While measurement errors can often be safely corrected, mechanistic anomalies require careful consideration. In this paper, we  formally define a causal model that explicitly captures both types by treating outliers as latent interventions on latent (“true”) and observed (“measured”) variables and show under which conditions the distinction is possible. Based on this model, we develop an efficient inference procedure for localizing root causes and distinguishing anomaly types. Experiments on synthetic and real-world data show that our method provides state-of-the-art and highly robust performance in both root cause localization and classification of anomaly types.

%% file: introduction.tex
\section{Introduction}
\label{sec:intro}

\input{figs/intro-plot.tex}

Real-world datasets are often noisy and contain anomalies---samples generated by processes that deviate from the norm. 
These anomalies can compromise model reliability, lead to biased conclusions, or misguide scientific and engineering insights. Importantly, not all anomalies are alike: some result from measurement or data collection errors, such as sensor malfunctions or misrecorded values, while others reflect genuine deviations in the underlying causal mechanisms, indicating that the process itself has changed. 
We call these fundamental types of anomalies \emph{measurement anomalies} and \emph{mechanistic anomalies}. In many applications, it is important to tell these two apart: an anomalous pressure reading in an industrial plant may indicate a faulty sensor, which calls for recalibration or replacement, or a genuine change in the underlying physical process, which may require operational intervention.

We give an example using the DAG $W \leftarrow X \rightarrow Y \rightarrow Z$ in Fig.~\ref{fig:intro}. In the left plot, points \g and \m stand out. For neither point is it immediately clear if this is due to a measurement error or change in the mechanism generating the point. We can gain more insight by examining additional causal relationships. The middle plot shows that neither point stands out for $X \rightarrow W$. The right-hand plot, however, reveals that point \m violates both $X \to Y$ and $Y \to Z$ in a way that is consistent with having measured too low a value for $Y$. In contrast, as the red point (\g) does follow the regular relationships between $X \to W$ and $Y \to Z$, this indicates this anomaly is most likely generated by a mechanism shift between $X \to Y$.

Explainable anomaly detection (XAD) typically relies on feature-attribution methods that explain an anomaly score with respect to a trained model, quantifying which inputs the model associates most strongly with anomalousness~\citep{ribeiro:2016:lime,lundberg:2017:shap}. Root cause analysis (RCA) of anomalies instead seeks to identify the variables that actually cause the deviation in the underlying data-generating mechanism, by incorporating causal assumptions and reasoning~\citep[]{budhathoki:2022:rootcause,han:2023:rootclam,orchard:2025:rootcause,nguyen:2024:noisyedge}.

However, existing methods have thus far blurred the distinction between measurement and mechanistic outliers, despite its practical importance. 
To bridge this gap, we propose a causal formalization of measurement and mechanistic anomalies. Our framework allows us to not only identify which variables caused the anomaly, but in addition, also determine whether an anomaly arises from a measurement error or a mechanism change.

Specifically, we interpret mechanistic and measurement anomalies as different interventions on the underlying causal graph $\Gc$. The key difference from existing RCA methods is that we consider interventions on latent ("real") and observed ("measured") variables, respectively. Similar to existing RCA methods \citep{orchard:2025:rootcause}, we leverage the Independence of Causal Mechanisms~\citep[][]{schoelkopf:2012:causal,peters:2017:elements}, and in particular the \textit{Sparse Mechanism Shift} hypothesis (SMS)~\citep[][]{schoelkopf:2021:towardsclr,perry:2022:sms}. This posits that explanations that require \emph{multiple} mechanisms to fail \emph{jointly} are less likely. In Fig.~\ref{fig:intro}, this implies that interpreting \m\ as a measurement outlier (one failure) is more plausible than concerted mechanism changes in both $X \to Y$ and $Y \to Z$.

Since clean training data is often unavailable in practice, we depart from existing methods and consider the unsupervised in-sample setting, where anomalies must be identified directly within the given dataset. Our contributions are threefold: (i) We formalize mechanistic and measurement anomalies within a causal framework.  (ii) We characterize identifiability of the reduced intervention probabilities, the global probabilities with which different mechanistic and measurement interventions occur. (iii) We develop a practical algorithm that takes as input a known or learned causal DAG, achieves state-of-the-art performance on synthetic and real-world datasets, and discovers meaningful anomaly types in a case study on real-world data.

%% file: related_work.tex
\section{Related Work}
\label{sec:related}

We situate our work within the literature on causal reasoning for outlier analysis; an overview of unsupervised outlier detection and feature attribution is provided in Appx.~\ref{apx:related_work}.

Root Cause Analysis (RCA) aims to identify the causal factors that most directly explain a target outcome of interest, ranging from treatment effects to system failures \citep{li:2022:rca,strobl:2023:samplerca,strobl:2023:hetero,strobl:2024:counterfactual,pham:2024:baro,li:2025:rca,ikram:2025:rca}.

Recent work has instantiated RCA for anomaly detection. \citet{budhathoki:2022:rootcause} assume a known structural causal model (SCM) and attribute feature-level contributions using Shapley values, which was adopted in subsequent RCA work \citep{strobl:2023:samplerca,strobl:2023:hetero}. \citet{han:2023:rootclam} propose a causal VAE for anomaly localization. \citet{orchard:2025:rootcause} assume a polytree-structured DAG and identify root causes via marginal outlier scores. However, these approaches primarily identify \emph{where} anomalies originate, without explaining \emph{how} they occur. We address this gap by extending the causal framework to jointly model the two principal categories of measurement and mechanistic outliers. 

In the time series literature, heuristic distinctions between “sensor” and “process” anomalies have been proposed \citep{salem:2013:sensor,haque:2015:sensor,bachechi:2022:detection,larosa:2022:separating}. These methods rely on correlation-based criteria or domain knowledge rather than an explicit causal formalization. The closest such approach distinguishes “cyber” and “measurement” anomalies using Granger causality \citep{sun:2025:unifying}, which captures predictive temporal dependencies rather than interventional causal structure. The method is further restricted to time series and assumes sufficient samples per failure mode. In contrast, our approach detects, localizes, and classifies heterogeneous failure types even with as little as a single sample per type on i.i.d. data.

%% file: preliminaries.tex
\section{Causal Model}
\label{sec:prelim}

\input{figs/example_dag}
Throughout, let $\Xt=(X^\ast_1,\ldots,X^\ast_d)$ denote latent causal state variables and $\Xb=(X_1,\ldots,X_d)$ their observed readouts. 
Most samples follow an anomaly-free process; a small fraction deviate from it and are called \textit{outliers} or \textit{anomalies}. 
Bold uppercase letters denote variable sets, and bold lowercase letters denote their samples.
\subsection{Anomaly-free Causal Model}

We first describe the anomaly-free process. 
Let $\Vb=\Xt\cup\Xb$ and let $\Gc=(\Vb,E)$ be a DAG. 
We write $\Pa_j^*\subset\Xt$ for the latent parents of $X_j^*$, $\Pa_j^{*} = \left\{ X_k^{*} : X_k^{*} \to X_j^{*} \in E \right\}$, and 
$\Pa_j:=\{X_k\in\Xb:X_k^\ast\in\Pa_j^*\}$ for their observed counterparts. 
Each $X_j$ has exactly one parent, $\Xtj$ (see Fig.~\ref{fig:chain}).

Unless stated otherwise, sink nodes are defined with respect to the latent subgraph $\Gc[\Xt]$, i.e., they are latent variables with no children in $\Xt$.
As standard causal modeling conditions, we assume causal sufficiency and faithfulness over $\Xt$~\citep{pearl:09:causalitybook}: 
\begin{enumerate}[noitemsep, topsep=0pt,leftmargin=*]
\item \emph{Causal Sufficiency}: There are no hidden confounders affecting two or more variables of $\Xt$.
\item \emph{Causal Faithfulness}: All independencies in the data are reflected by $d$-separation in the graph.
\end{enumerate}

In the anomaly-free regime, we model the data by a structural causal model~\citep[SCM;][]{pearl:09:causalitybook}. 
Each latent $\Xtj$ is generated by a structural assignment with independent noise $U_j$, and each observation is a faithful measurement of the corresponding latent state, i.e. $\Xtj = f_j(\Paj^*, U_j)$, $X_j = \Xtj.$
Under this clean model, the joint distribution factorizes in causal order as
\begin{align*}
  P(\Xt,\Xb) 
  = \prod_{j=1}^d P(\Xtj \mid \Paj^*)\; P(X_j \mid \Xtj)\,.
\end{align*}
Thus, in the absence of anomalies, the observed $X_j$ accurately reflect the latent causal process encoded by $\Gc$.
Anomalies are modeled as sparse sample-level deviations from this clean process.

\subsection{Mechanistic and Measurement Anomalies}
We model two kinds of deviations from the clean process: perturbations of latent causal mechanisms and perturbations of their observed measurements.

\begin{definition}[Mechanistic Anomalies]
A mechanistic anomaly at variable $j$ is modeled by replacing the clean structural assignment of $\Xtj$ with an outlier variable $\GenRVj \sim P_{\GenRVj}$. 
With latent indicator $\Gen_j \sim \mathrm{Bernoulli}(\GenPj)$,
\begin{align*}
  \Xtj 
  = (1 - \Gen_j) f_j(\Paj^*, U_j) 
    + \Gen_j \, \GenRVj,
  \qquad 
  \gen_j \in\{0,1\}\,. 
\end{align*}
When $\Gen_j=1$, the latent causal state itself is perturbed, and the effect may propagate.
\end{definition}

\begin{definition}[Measurement Anomalies]
A measurement anomaly at variable $j$ is modeled by replacing only the observed readout $X_j$ with an outlier variable $\MeasRVj \sim P_{\MeasRVj}$. 
With latent indicator $\Meas_j \sim \mathrm{Bernoulli}(\MeasPj)$,
\begin{align*}
  X_j  
  = (1 - \Meas_j) \Xtj 
    + \Meas_j \MeasRVj,
  \qquad 
  \meas_j \in\{0,1\}\,.
\end{align*}
When $\Meas_j=1$, the latent causal state $\Xtj$ remains unchanged, so the perturbation does not propagate.
\end{definition}

The replacement variables \(\GenRVj\) and \(\MeasRVj\) can be viewed as additional exogenous variables, analogous to the usual noise variables in an SCM. We take \((\mathbf \GenRV,\mathbf \MeasRV)\) to be independent of the clean SCM variables and noises, and do not represent them explicitly in the DAG. We rather interpret anomalies as hard interventions~\citep{pearl:09:causalitybook} on the DAG: a mechanistic anomaly \(\Gen_j=1\) replaces the clean latent mechanism \(f_j(\Pa^*_j,U_j)\) by an exogenous draw from \(\GenRVj\), while a measurement anomaly \(\Meas_j=1\) replaces the clean measurement mechanism \(X_j=X_j^*\) by an exogenous draw from \(\MeasRVj\).
We assume that anomalies are rare, meaning that the clean assignment occurs for the majority of samples, i.e., $P((\Measb,\Genb)=\mathbf 0)\gg 0.5.$

Each sample \(\xbi\) can therefore be associated with an assignment-specific
graph \(\Gci\). For non-anomalous samples this graph is simply
\(\Gc\). For anomalous samples, the edges \(\Paj^*\to \Xtj\) or
\(\Xtj\to X_j\) are removed according to the anomaly assignment; see
Fig.~\ref{fig:ex_meas} and Fig.~\ref{fig:ex_gen}.

An alternative would be to model anomalies as soft interventions~\citep{pearl:09:causalitybook}. In contrast
to the hard interventions used here, they replace
the clean mechanism by a different mechanism that may still depend on those
parents. For example, a mechanistic anomaly could be modeled as $
  \Xtj := (1-\Gen_j) f_j(\Pa^*_j,U_j)
     + \Gen_j g_j(\Pa^*_j,U_j).$
We do not model such soft interventions because anomalies are rare: each failure type
typically provides too few samples to reliably estimate a separate parent-dependent
mechanism \(g_j\) in practice. Without prior knowledge of the failure mechanism, or enough
repeated samples from the same failure mode as in causal mixture
models~\citep{liu:2015:causal,hu:2018:causal}, soft interventions would add
flexibility (variance) without reliable statistical support. We therefore use hard
interventions as a minimal model of violations of the usual parent-conditional mechanism. 

The key distinction between the two anomaly types lies in propagation. 
Because observed variables $X_j$ have no children, a measurement anomaly $\meas_j=1$ affects only the observed coordinate $X_j$: the underlying latent state $\Xtj$ and all downstream variables remain unchanged. 
By contrast, a mechanistic anomaly $\gen_j=1$ replaces the mechanism generating $\Xtj$ itself, and therefore can affect the marginal distributions of its descendants. 
For instance, in Fig.~\ref{fig:ex_meas}, $X_2$ is replaced at the measurement level, while the latent causal relations 
$X^\ast_1 \to X^\ast_2 \to X^\ast_3$ remain intact. 
In Fig.~\ref{fig:ex_gen}, the mechanism generating $X^\ast_2$ is replaced, thereby perturbing the downstream variables as well.

Finally, we instantiate the sparse mechanism shift principle~\citep[SMS;][]{schoelkopf:2021:towardsclr} through an independent Bernoulli prior over the latent anomaly indicators, i.e. we model $(\Measb,\Genb)$ as jointly independent.
This acts as an inductive bias toward simple explanations: we prefer explanations that require fewer root causes, unless the evidence strongly suggests otherwise.

In the next section, we study inference under this model. In particular, we first analyze which intervention probabilities are identifiable once the assignment-conditional densities are fixed, and derive a tractable approximation for inferring root causes and anomaly types from observed data.

%% file: theory2.0.tex
\section{Inference of Root Causes and Anomaly Types}
\label{sec:theory}
We formalize inference under the proposed causal anomaly model. 
We first show how assignment-conditional densities are evaluated, including the marginalization required by measurement anomalies. 
We then characterize natural equivalences between assignments and analyze when, given fixed assignment-conditional densities, the reduced intervention probabilities are identifiable.
Finally, since exact marginalization over assignments is computationally infeasible, we introduce a tractable hard-assignment approximation.

In finite samples, anomalies are often too rare to estimate separate replacement distributions for all failure types. 
We therefore specify simple symmetric outlier distributions, \(\tilde P_{\GenRVj}= \tilde P_{\MeasRVj}\), so that the model does not favor one anomaly type a priori. 
For the theoretical analysis, we condition on a given DAG and anomaly-free SCM density. Together with the fixed outlier distributions, these determine the assignment-conditional densities used below. 
We defer the practical estimation of the anomaly-free density from contaminated data to Sec.~\ref{sec:algo}.

\paragraph{Assignment-conditional densities}
We now derive the assignment-conditional density \(p(\xb\mid \ass)\), the density of a sample $\xb$ under a given assignment \(\ass=(\genb,\measb)\).
For outlier indicators $\ass=(\genb,\measb)$ of a sample $\xb$, define
$\Genc_\ass=\{j:\gen_j=1\}$ and $\Mc_\ass=\{j:\meas_j=1\}$.
Let \(\bar x_j\) be the unobserved value before measurement replacement if \(j\in\Mc_\ass\), and set \(\bar x_j=x_j\) otherwise.
Let \(\bar{\pa}_j\) denote the parent values of node $j$ in \(\bar{\xb}\). %\note{what does this mean? I think it might be cleaner if we just refer to it as the relevant subset of $\bar x$?}
Marginalizing the latent values of measurement outliers yields
\begin{align}
  p(\xb\mid \ass)   
  =
  \prod_{j \in \Mc_\ass} \tilde p_{\MeasRVj}(x_j)
  \int
  \left[
  \prod_{j \notin \Genc_\ass}  p(\bar x_j \mid \bar{\pa}_j)
  \prod_{j \in \Genc_\ass} \tilde p_{\GenRVj}(\bar x_j)
  \right]
  \prod_{j \in \Mc_\ass} d\bar x_j .
  \label{eq:margin}
\end{align}
We provide the detailed derivation of this expression in Appx.~\ref{apx:derivation}.
In general, the integral cannot be computed in closed form but can be approximated via Monte Carlo sampling (Sec.~\ref{sec:algo}).
The joint density of $\xb$ and $\ass$ is obtained by multiplying the assignment-conditional density by the prior probability of the assignment:
\[
p(\xb ,\ass \mid \pib)
= p(\xb\mid \ass)P(\ass\mid \pib)= p(\xb\mid \ass)  \prod_{j=1}^d \mathrm{Bern}(\gen_j; \GenPj)\,\mathrm{Bern}(\meas_j; \MeasPj).
\]

Since assignments can induce identical observable densities, we define an equivalence relation.

\begin{definition}[Equivalent assignments]
Two assignments $\ass,\ass'$ are called equivalent, written $\ass \equiv \ass'$, if they induce the same assignment-conditional density, i.e.,
$
p(\xb\mid \ass) = p(\xb\mid \ass') \quad \text{for all } \xb.
$
\end{definition}

The most common equivalences arise at sink nodes. %\note{most important? most common? simplest?}
If \(j\) is a sink and \(P_{\GenRVj}=P_{\MeasRVj}\), then a mechanistic intervention on \(X_j^\ast\) and a measurement intervention on \(X_j\) induce the same observable density: neither can propagate downstream, and both replace only the observed coordinate \(X_j\). 
Thus, for any fixed intervention pattern on the remaining nodes,
$
(\gen_j,\meas_j)\in\{(1,0),(0,1),(1,1)\}
$
are equivalent, and we represent them by
$
s_j:=\gen_j\vee\meas_j.
$ 
Further equivalences can arise for multi-intervention assignments, since one intervention may remove dependencies through which another intervention would otherwise be distinguishable.

\paragraph{Inference of root causes and anomaly types}
Under fixed intervention probabilities \(\pib\), the sample-conditional assignment posterior is given by \[q_{\pib}(\ass\mid\xb) =
\frac{p(\xb,\ass\mid \pib)}
{\sum_{\ass'} p(\xb,\ass'\mid \pib)}\,.\]  
The Bayes-optimal assignment classifier is the maximum-a-posteriori assignment $\argmax_{\ass} q_{\pib}(\ass\mid\xb) = \argmax_{\ass}   \log p(\xb,\ass\mid \pib)$. Importantly, if assignment-conditional distributions overlap on a set of positive probability even the Bayes-optimal classifier has nonzero error. For example, a very unlikely statistical outlier that was generated by the anomaly-free SCM may be misclassified as a mechanistic anomaly.
In practice, intervention probabilities are typically unknown. A natural
reference objective is the marginalized likelihood
\begin{align}
    \max_{\pib} \sum_{i=1}^n \log \left( \sum_{\ass} p(\xbi,\ass \mid \pib) \right),
    \label{eq:mle}
\end{align}
which corresponds to a mixture model over fixed assignment-conditional densities. Since different assignments may induce the same observable density, we can group assignments into structural equivalence classes and index mixture components by these classes. This rewrites \eqref{eq:mle} as a standard mixture model with aggregate weights, where each aggregate weight is the total prior mass of all assignments in the corresponding equivalence class.
Because of sink-node equivalence, we use the reduced indicator
\(s_j:=\gen_j\lor\meas_j\) for \(j\in\mathcal S\), where \(\mathcal S\)
denotes the sink nodes of \(G[X^*]\), and define
\(\rho_j:=P(s_j=1)\). Thus, \(\pib_{\rm red}\) collects
\((\GenPj,\MeasPj)_{j\notin\mathcal S}\) and \((\rho_j)_{j\in\mathcal S}\).
In Appx.~\ref{apx:proofs}, we show that, given a suitable subset of
uniquely identifiable aggregate mixture weights, the reduced Bernoulli
probabilities \(\pib_{\rm red}\) can be recovered. In particular,
\(\GenPj\) and \(\MeasPj\) are separately identifiable only for non-sink
nodes, while at sink nodes only \(\rho_j=P(\gen_j=1\lor\meas_j=1)\) is
identifiable. 

We show generic identifiability of \(\pib_{\rm red}\) for linear
Gaussian SCMs (Appx.~\ref{apx:linear-gaussian}) and for fixed-degree
polynomial Gaussian SCMs (Appx.~\ref{apx:polynomial-gaussian}). 
For conciseness, we here state the resulting generic identifiability guarantee in simplified form; the full statement and proof are given in Appx.~\ref{apx:polynomial-gaussian}.
\begin{theorem}[Generic identifiability under polynomial Gaussian SCMs (simplified)]
\label{thm:main-poly-identifiability}
Under the setup established above, suppose the clean mechanisms are additive-noise
polynomial functions of fixed degree \(r\ge 1\), the exogenous noises are
Gaussian with positive variances, and the mechanistic and measurement
outlier distributions are nodewise identical Gaussians. Then, outside a
Lebesgue-null set of model parameters, the reduced intervention probabilities
\(\pib_{\rm red}\) are identifiable.
\end{theorem}

Although identifiable, optimizing Eq.~\eqref{eq:mle} is generally infeasible due to the exponential number of assignments, many of which require marginalizing latent pre-replacement values. 
Moreover, in the sparse finite-sample regime, many low-probability assignments are observed rarely or not at all, making their aggregate mixture weights difficult to estimate reliably. 
We therefore use Eq.~\eqref{eq:mle} as a reference objective, but adopt a hard-assignment approximation for finite-sample inference. 
Specifically, we approximate the sum over assignments by a maximum:
\[
\log \sum_{\ass} p(\xb,\ass\mid \pib)
\approx
\max_{\ass} \log p (\xb,\ass\mid \pib)\,.
\]
The gap is controlled by posterior uncertainty. We have %$\note{perhaps the entropy should be over rvs $A, X$ rather than values?}
\[
0\le
\log\sum_{\ass}p(\xb,\ass\mid \pib)
-
\max_{\ass}\log p(\xb,\ass\mid \pib)
\le
H\left({q_{\pib}}(\cdot \mid\xb)\right) \; 
\]
Here, \(H\left({q_{\pib}}(\cdot \mid\xb)\right)=-\sum_{\ass}q_{\pib}(\ass\mid\xb)\log q_{\pib}(\ass\mid\xb)\) denotes the Shannon entropy of the assignment posterior \(q_{\pib}(\cdot\mid\xb)\).
Thus, the approximation is tight when the posterior over assignments is concentrated, which we expect when outliers are strong enough to induce well-separated assignment-conditional densities.
The resulting hard-assignment objective remains combinatorial and non-convex. 
For fixed assignments, the optimal intervention probabilities are their empirical frequencies; our algorithm therefore profiles them out and uses the induced Bernoulli term as an adaptive sparsity cost, as described in Sec.~\ref{sec:algo}.
In Sec.~\ref{sec:exps}, we show that the resulting approximation closely matches the optimal assignment classifier with oracle $\pib$ where exact inference is feasible, while scaling to real-world problems.

%% file: instantiation.tex
\section{Practical Inference with \ourmethod}
\label{sec:algo}

We now turn the theoretical formulation into a practical algorithm that takes the observed data \Xb{} and a learned or user-specified DAG $\hat \Gc$ as inputs.
We first specify how the anomaly-free SCM density and outlier distributions are estimated. 
We then describe likelihood evaluation under a fixed assignment, including Monte Carlo marginalization for measurement anomalies. 
Finally, we present \ourmethod,\!\footnote{\textbf{C}ausal \textbf{A}nomaly \textbf{L}ocalization and \textbf{I}nspection} a practical algorithm that jointly updates assignments and intervention probabilities. We present the high-level ideas and algorithms here, while further details, pseudocode, and complexity analysis are given in Appx.~\ref{apx:algorithm}. The source code, including the implementation of the proposed algorithm and all experimental evaluations, is publicly available on GitHub\footnote{\url{https://github.com/hsuhr99/measurement_mechanistic_rca}}.

\paragraph{Anomaly-free density.}
Although the anomaly-free density could be instantiated with arbitrary causal models, we use an additive noise model (ANM) for tractable conditional-density estimation under contamination. 
Given $\hat \Gc$, each node is modeled as $
X_j = f_j(\Pa_j) + \varepsilon_j.
$
We estimate $\hat f_j:\mathbb R^k\to\mathbb R$ using generalized additive models with cubic spline bases~\citep[GAMs;][]{hastie:2017:generalized}. 
To reduce sensitivity to anomalies, we fit the mechanisms by median regression using the quantile-regression approach of \citet{fasiolo:2021qgam}. 
Residuals are then defined as
$
r_j = x_j-\hat f_j(\pa_j).
$
We estimate their densities using trimmed kernel density estimation: a fraction $\alpha$ of observations farthest from the median is discarded before fitting a Gaussian KDE~\citep{parzen:1962:estimation}. Together, the fitted mechanisms \(\{\hat f_j\}_j\) and residual densities define
the estimated anomaly-free conditionals
\(\hat P(X_j\mid \Pa_j)\).

\paragraph{Assignment likelihoods.}
For the replacement distributions, we use uniform distributions over the observed feature ranges. 
Given an assignment, likelihood evaluation follows Eq.~\eqref{eq:margin}. 
When measurement anomalies are present, the pre-replacement values are latent; we approximate the corresponding integral by Monte Carlo. 
For each $j\in\Mc$, we sample candidate clean values from $\hat P(X_j\mid \Pa_j=\pa_j)$, evaluate the remaining factors in Eq.~\eqref{eq:margin}, and average over samples. 

\paragraph{Assignment optimization.}

For exact assignment classification under fixed $\pib_\text{red}$, used only in small oracle comparisons, we enumerate all $4^d$ assignments for each sample  and choose the maximizer of $p(\xb,\ass \mid \pib_\text{red})$. In contrast, \ourmethod approximates the profiled hard-assignment objective heuristically. 
We initialize all samples with the all-clean assignment and profile $\pib_\text{red}$ by empirical assignment frequencies. 
Candidate interventions are prioritized by low residual likelihood under the clean density. 
A candidate is accepted only when the resulting gain in assignment-conditional likelihood outweighs the induced increase in Bernoulli sparsity cost. 
This procedure iterates until no candidate update improves the objective. 
It reduces the search from exhaustive enumeration over assignments to a greedy scan over feature-level candidates. 
In Sec.~\ref{sec:exps}, we show that \ourmethod closely matches exact classification with oracle $\pib_\text{red}$, where enumeration is feasible, while scaling to real-world problems.

\paragraph{Root-cause scores and type confidence.}
\ourmethod returns a discrete assignment \(\ass\) for every sample, but it is often useful to rank
candidate root causes by plausibility. 
Since exact posteriors are intractable, we use a likelihood-gain surrogate. 
For an assignment $\ass$, let $\ass[j \leftarrow 0]$ denote the assignment obtained by setting $(\gen_j,\meas_j)=(0,0)$, and let $\ass[j \leftarrow \mathrm{mech}]$ and $\ass[j \leftarrow \mathrm{meas}]$ denote the assignments obtained by setting $(\gen_j,\meas_j)=(1,0)$ and $(0,1)$, respectively, leaving all other indicators unchanged. 
Let 
$\ell_j^0 := \log p(\xb,\ass[j \leftarrow 0]\mid\pib)$, $\ell_j^{\mathrm{mech}} := \log p(\xb,\ass[j \leftarrow \mathrm{mech}]\mid\pib)$, and $\ell_j^{\mathrm{meas}} := \log p(\xb,\ass[j \leftarrow \mathrm{meas}]\mid\pib)$.
We define the root-cause score as $
  \Delta_j := \max\{\ell_j^{\mathrm{mech}},\ell_j^{\mathrm{meas}}\} - \ell_j^0,
  \label{eq:scores}
$
where $\Delta_j$ measures the best likelihood gain obtained by treating $j$ as a root cause. 
This score is especially useful because the uniform replacement distribution is only a simple approximation, which is likely misspecified in practice. Relative likelihood gains can still rank plausible root causes highly, even when conservative replacement likelihoods make the absolute gains small or negative.
We also define the type confidence as
$
\tilde C_j
:=
\left|
\ell_j^{\mathrm{mech}}
-
\ell_j^{\mathrm{meas}}
\right|.
\label{eq:conf}
$
It measures approximately how strongly the model favors one anomaly type over the other, conditional on $j$ being a root cause.

%% file: experiments.tex
\section{Experiments}
\label{sec:exps}

\usetikzlibrary{patterns}

We empirically evaluate \ourmethod across various settings.
In addition to the main experiments presented here, we evaluate detection performance, sensitivity to sample size and dimensionality, runtime, and robustness to violations of the modeling assumptions, including correlated root causes, soft interventions, controlled DAG misspecification, confounding, heteroscedastic noise, extreme outliers, global measurement noise, and endogenous intervention probabilities in Appx.~\ref{apx:experiments}.% in Appx.~\ref{apx:violations}. \note{since 9.2 is just a subsection, we can also just refer to all of these as extra experiments in 9}

\begin{wrapfigure}[15]{r}{0.5\textwidth}
    \centering
    \vspace{-0.5cm} % optional

    % Legend
    \begin{subfigure}[t]{\linewidth}
        \centering
        \input{figs/rca_small_legend.tex}
    \end{subfigure}

    \vspace{0.5em}

    % Two plots
    \begin{subfigure}[t]{0.49\linewidth}
        \centering
        \input{figs/rca_small.tex}
    \end{subfigure}
    \hfill
    \begin{subfigure}[t]{0.49\linewidth}
        \centering
        \input{figs/classification_small.tex}
    \end{subfigure}

    \caption{\ourmethod closely matches exhaustive model-optimal assignment classification. 
Root-cause localization (Top-\(k\) recall, left) and anomaly-type classification accuracy (right) on synthetic data, compared against the model-optimal classifier obtained by exhaustive enumeration over assignments.}
%   \vspace{4em}   
\label{fig:oracle}
   
    % \vspace{-\baselineskip} % optional
    % \hspace{0.5em} % optional
\end{wrapfigure}
\paragraph{Data} 
To be able to evaluate on data with known ground truth, we generate $2000$ samples $\mathbf{x}$ from an SCM defined over a random Erd{\"o}s--R{\'e}nyi DAG with $15$ nodes. Each variable \(X_j\) is specified as $X_j = f_j\bigl(\mathrm{Pa}(X_j)\bigr) + U_j$, where \(f_j\) is a quadratic or cubic polynomial and \(U_j \sim \mathcal{N}(0, \sigma^2)\). For 10\% of samples, outliers are injected by shifting the marginal distribution of \(X_j\)  by \(l\) empirical standard deviations and replacing the actual value with a random sample from the shifted marginal distribution, and recalculating downstream values for mechanistic outliers. We refer to $l$ as \emph{outlier strength}. 

In addition, we evaluate on two real-world datasets with known consensus DAGs: the Sachs dataset~\citep{sachs:2005:causal} and the \emph{Causal Chambers} benchmark~\citep{gamella:2025:causalchamber}. Both contain observational and interventional data. In Causal Chambers, interventions are applied at different strengths (\emph{mid} and \emph{strong}), whereas interventions in the Sachs dataset induce weaker marginal shifts than the \emph{mid} interventions. We simulate mechanistic outliers by injecting a small fraction of interventional samples, and measurement outliers by replacing individual features with values drawn from the post-intervention distribution. We consider three settings with respect to outlier strength: low (Sachs), medium (Causal Chambers, \emph{mid}), and high (Causal Chambers, \emph{strong}). Further details on data-generation and processing are provided in Appx.~\ref{apx:datasets}.

\paragraph{Instantiating \ourmethod}
To evaluate the best possible performance, we consider \ourmethod with access to the true causal graph $\Gc^*$. Since the ground truth is often unknown, we also consider the case where \ourmethod is given a learned causal DAG. To this end, we apply \ourmethod to the outputs of various causal discovery algorithms on independently generated synthetic datasets and assess downstream performance (see Appx.~\ref{apx:causal_discovery}). The best performance occurs with \NOGAM~\citep{montagna:2023:causal}, with \AUTOENCODER-based prefiltering, which we thus use for subsequent experiments. We write \ourmethodstar for when it has access to the true graph, and \ourmethodprime when the causal graph is learned. 

\paragraph{Evaluation Metrics}
For root cause localization, we report Top-$k$ Recall, where $k$ is set dynamically to the true number of root causes of the respective sample, which calculates what percentage of the true root causes is included in the top $k$ candidates. This metric is detection-agnostic and ensures a fair comparison across methods. However, it does not necessarily reflect performance in deployment, as explaining anomalies that are not detected provides no actionable insight. Therefore, we also report an evaluation that takes detectability into account in Appx.~\ref{apx:metrics}, which also includes further metrics. 
For the classification of outlier types, we report the accuracy over correctly detected root causes,  evaluating the classification performance for actionable anomalies. The shaded regions in figures show mean \(\pm 1.96\) standard errors across 20 independent runs.

\subsection{Comparison to Model-Optimal Assignment Classification}

To evaluate whether the approximations used by \ourmethod affect performance in practice, we compare it to exhaustive model-optimal assignment classification under oracle intervention probabilities \(\pib_{\mathrm{red}}\), denoted \OPTCLASS. \OPTCLASS is optimal only with respect to the estimated likelihood, not the true data-generating distribution. Since it enumerates all assignments, it is exponential in \(d\), so we restrict this comparison to the synthetic setting with 500 samples and 3 nodes. Fig.~\ref{fig:oracle} shows that \ourmethodstar closely matches \OPTCLASS in both root-cause localization and anomaly-type classification accuracy, while being substantially more efficient (47.57s vs.\ 1190.60s on average). This suggests that the approximations introduce little loss in the settings where exact comparison is possible.

\subsection{Root Cause Localization}
\begin{wrapfigure}[11]{r}{0.5\textwidth}
    \centering
    \vspace{-1.3cm} % optional­

    % Legend
    \begin{subfigure}[t]{\linewidth}
        \centering
        \input{figs/top_k_recall_legend.tex}
    \end{subfigure}

    \vspace{0.3em}

    % Two plots
    \begin{subfigure}[t]{0.49\linewidth}
        \centering
        \input{figs/top_k_recall}
    \end{subfigure}
    \hfill
    \begin{subfigure}[t]{0.49\linewidth}
        \centering
        \input{figs/rca_cc_top_k_recall}
    \end{subfigure}

    \caption{Root Cause Localization (Top-$k$ recall, with $k$ set to the true number of root causes in each sample) for synthetic data (left) and real data (right). * denotes ground-truth DAG access.
    }
    \label{fig:rca}
    \vspace{-\baselineskip} % optional
    % \hspace{0.5cm} % optional
\end{wrapfigure}%
To investigate how well \ourmethod performs compared to existing root cause analysis methods, we grant all methods access to the ground-truth DAG where applicable and mark such methods with a *. We consider \SCM an adapted version of SCM-based attribution approaches \citep{budhathoki:2022:rootcause,strobl:2023:samplerca,strobl:2023:hetero}, \SMOOTHTRAVERSALSTAR and \SCOREORDER \footnote{Short for  \textsc{SmoothTraversal}, \textsc{ScoreOrder}}~\citep{orchard:2025:rootcause}, and  \AESHAP (autoencoder  with Shapley values). We refer to Appx.~\ref{apx:implementation} for detailed implementations and hyperparameter selection.  
Since all competitors assume a given split between the normal and anomalous regimes, and are thus not directly applicable in the in-sample setting, we combine them with the best-performing sample detector aside from \ourmethod to obtain the split. 

We report the results in Fig.~\ref{fig:rca}. On synthetic data, \ourmethodstar achieves the best performance, with \ourmethod using a learned graph as a close second. On the real-world datasets, all methods struggle under the weak outliers in the Sachs dataset, resulting in uniformly low Top-$k$ recall. In contrast, on the Causal Chambers dataset, \ourmethodstar already attains a Top-$k$ recall of 90\% under the medium outlier strength, a level that is only matched by a single competitor under strong outliers. 
A surprising observation is that \AESHAP is very competitive for strong interventions on the Causal Chamber dataset. It performs on par with \ourmethodstar, and outperforms \ourmethod. 

One likely reason why \ourmethod performs worse than \ourmethodstar is the causal discovery step. We provide causal discovery metrics for both the original setting and the clean, anomaly-free Causal Chamber dataset in Tab.~\ref{tab:cc-causal-discovery-diagnostics} in Appx.~\ref{apx:metrics}. The learned graphs remain substantially different from the reference graph even when discovery is performed on clean observational data, suggesting that the performance gap between \ourmethod and \ourmethodstar primarily reflects limitations of causal discovery on this dataset, rather than contamination alone. 
Another factor for the good performance of \AESHAP may be the graph structure of the Causal Chamber dataset. The DAG is very shallow, with maximum depth 2, which reduces the penalty for ignoring causal propagation structure and gives association-based methods an unusually high chance of identifying the correct root causes.

Importantly, our evaluation measures root-cause localization under oracle detection. When we take detectability into account (see Appx.~\ref{apx:additional_rca}), \ourmethodstar and \ourmethod outperform competing approaches by an even larger margin.

\subsection{Anomaly Type Classification}

Next, we evaluate whether \ourmethod can accurately classify the types of detected root causes. As a baseline, we consider \MARGINAL, which identifies outlier types using marginal outlierness. A root cause is classified as a measurement outlier if none of its child values exceed a predefined outlierness threshold. A key limitation of this baseline is the difficulty of threshold selection. Since we want to eliminate errors in causal discovery for the comparison, we only consider \ourmethodstar and \MARGINALSTAR for this experiment. We provide the results with learned graphs in Appx.~\ref{apx:classification}. To ensure a fair comparison, we provide \MARGINALSTAR with the root causes identified by \ourmethodstar and tune its significance threshold on synthetic data. Detailed implementation and hyperparameter selection are described in Appx.~\ref{apx:implementation}. 

Fig.~\ref{fig:classification} shows that \ourmethodstar outperforms \MARGINALSTAR. Despite the poor localization performance on the Sachs dataset, those root causes that are correctly identified are classified with high accuracy. Especially under weak outliers \ourmethod beats \MARGINAL by a large margin. Another major advantage of \ourmethod is that it performs robustly across different settings without requiring a threshold, as we further illustrate in the "influence of degree" experiment in Appx.~\ref{apx:classification}.

\subsection{Case Study on NYC Taxi Data}
\begin{wrapfigure}[11]{r}{0.5\textwidth}
    \centering
    \vspace{-1.4cm} % optional

    % Legend
    \begin{subfigure}[t]{\linewidth}
        \centering
        \input{figs/classification_legend.tex}
    \end{subfigure}

    \vspace{0.3em}

    % Two plots
    \begin{subfigure}[t]{0.49\linewidth}
        \centering
        \input{figs/classification_synth.tex}
    \end{subfigure}
    \hfill
    \begin{subfigure}[t]{0.49\linewidth}
        \centering
        \input{figs/classification_cc.tex}
    \end{subfigure}

    \caption{Classifying measurement and mechanistic anomalies. Accuracy on synthetic data (left) and real-world data (right). Random guessing corresponds to 50\%.
    }
    \label{fig:classification}
    \vspace{-\baselineskip} % optional
    % \hspace{0.5em} % optional
\end{wrapfigure}

To qualitatively verify the anomalies that \ourmethod finds in real data, we perform a case study on the NYC Yellow Taxi Trip Dataset \citep{elemento:2015:nyc}, which contains 12.7 million taxi trips with pickup and dropoff coordinates, times, fares, tips, extra charges, and the applied rate. We run \ourmethod on a random subset of 5\,000 entries using a selected set of continuous features. We illustrate the DAG used in the analysis in Fig.~\ref{fig:case_study_dag} in the Appx.~\ref{apx:datasets}, and the examples identified by \ourmethod in Tab.~\ref{tab:outliers_combined}. By filtering the root causes by types, and sorting based on gain and confidence we selected 3 anomalies with high likelihood gain and high classification confidence. 
\begin{table}[t]
   \caption{
Top 3 measurement (left) and mechanistic (right) anomalies identified by \ourmethod on the NYC Taxi dataset, ranked by confidence. Pick-up and drop-off coordinates are summarized as air distance, and detected root causes are shown in \textbf{bold}. 
The second measurement example has a \emph{Distance} value that is clearly incompatible with physical constraints, since it would imply a speed exceeding three times the speed of sound. All remaining attributes indicate a very short trip, supporting the measurement-error interpretation. For the second and third mechanistic examples, \ourmethod correctly detects the anomalous pricing mechanism corresponding to the fixed JFK flat fare of \$52.
}
\centering
% \small
 
\begin{subtable}[t]{0.49\textwidth}
     \resizebox{\linewidth}{!}{
    \centering
    \tiny
    \begin{tabular}{rrrrrr}
    \toprule
    Distance & Air Dist & Time & Fare & Tip & Total \\
    \midrule
    3.90 & 3.16 & \textbf{94.65} & 10.00 & 2.00 & 13.30 \\
    \textbf{9.40} & 0.01 & 0.20 & 2.50 & 0.00 & 3.30 \\
    0.60 & 0.98 & \textbf{33.52} & 5.00 & 0.00 & 6.30 \\
    \bottomrule
    \end{tabular}
    }
    % \caption{Measurement outliers}
    \label{tab:meas_outliers_cd}
\end{subtable}
\hfill
\begin{subtable}[t]{0.49\textwidth}
    \centering
    \resizebox{\linewidth}{!}{
        \tiny
    \begin{tabular}{rrrrrr}
    \toprule
    Distance & Air Dist & Time & Fare & Tip & Total \\
    \midrule
    0.00 & 0.06 & 0.75 & \textbf{220.00} & 44.05 & 264.35\\
    0.00 & 0.00 & 0.30 & \textbf{52.00} & 0.00 & 58.13\\
    1.80 & 1.39 & 8.48 & \textbf{52.00} & 11.63 & 69.76\\
    \bottomrule
    \end{tabular}
    }
    % \caption{Mechanistic outliers}
    \label{tab:mech_outliers_cd}
\end{subtable}

\label{tab:outliers_combined}

\end{table}

First, we look into the measurement outliers in Tab.~\ref{tab:outliers_combined} (left).
The first and third examples have recorded travel times of 94 and 33 minutes and distances of 3.9 and 0.6 miles, respectively, but especially the fares ($10$ and $5$) are suspiciously low compared to what one would expect from accrued time-based charges alone.
This inconsistency between distance, time, and fare, makes an incorrectly recorded travel time by far the most plausible explanation. The second example is perhaps even more clearly a measurement error: the trip distance of $9.4$ miles is incompatible with the pickup and drop-off coordinates, time, and fare, all of which instead indicate a very short trip. 

Next, we consider the mechanistic outliers in Tab.~\ref{tab:outliers_combined} (right).
The first example has a fare that is unusually high for a trip of less than a minute and zero-distance, yet consistent with the total amount and tip, reflecting a possibly fraudulent transaction. The second and third samples are textbook definitions of a mechanistic outlier as metadata indicates that these are rides under the special JFK flat fare, which sets a fixed fare of \$52 independent of distance or time. 

Overall, we find that \ourmethod successfully identifies measurement and mechanistic outliers in real-world data, with explanations that are plausible and actionable.

%% file: conclusion.tex
\section{Conclusion}
\label{sec:conclusion}

We studied root cause analysis from a new perspective, distinguishing measurement errors from genuine mechanistic anomalies by modeling both as interventions on latent and observed variables. We cast the problem as latent-variable maximum likelihood estimation and introduced a practical algorithm that achieves competitive root-cause localization while additionally enabling classification of outlier types. Finally, a case study on real-world data demonstrated that our approach produces plausible and actionable anomaly characterizations.

\paragraph{Limitations}
Our formulation reveals that mechanistic and measurement anomalies are not always distinguishable. In particular, at sink nodes, the two types of anomalies have identical observable effects and cannot be distinguished even with infinite data. For non-sink nodes, the right classification can also not be guaranteed since even under optimal classification there is a nonzero probability of misclassification when assignment-conditional distributions overlap. However, our experiments and especially the case study show that in practice, our method can provide useful and actionable insights into the nature of anomalies.

\ourmethod requires a DAG, and performance depends on how well the true mechanisms, noise model, and outlier types fit our modelling assumptions. In particular, the formal model treats anomalies as exogenous hard interventions and does not model ordinary non-anomalous measurement error. Empirically, \ourmethod remains robust to several violations of these assumptions, but performance can degrade when the model is severely misspecified. Our model also assumes a clean separation between measurements and the underlying causal system. If a recorded sensor value is fed into a downstream causal mechanism, this assumption is violated: the measurement itself becomes part of the causal system, rather than a pure leaf readout of an underlying latent state.

To enable computational tractability, we have to approximate the marginalized distribution with a hard approximation and can not provide exact posteriors. Furthermore, if the DAG is unknown, it must be estimated using a causal discovery method, whose limitations may be exacerbated by contamination. \ourmethod is currently restricted to tabular data, although its main ideas could be extended to other data types.

\paragraph{Future Work}
A promising direction is causal discovery under contamination. Rather than treating all anomalies as nuisance samples, future work could exploit their structure: measurement outliers may distort causal signals, whereas mechanistic outliers can act as informative perturbations. Distinguishing these cases may therefore benefit structure learning from contaminated data and help close the gap between \ourmethod and \ourmethodstar.

Another interesting direction is to amortize anomaly assignment inference using prior-fitted networks (PFNs) \citep{mueller:2021:pfns}. Rather than relying on the current hard-assignment approximation, one could train a transformer on synthetic causal models and anomaly-generation tasks to predict a distribution over plausible anomaly assignments. Such a model could incorporate assignment uncertainty directly, avoid explicit enumeration, and reduce dependence on a single learned graph by conditioning on the causal information that can be inferred from the observed data. This would allow the model to use causal structure when it is identifiable, while reflecting both structural and assignment uncertainty when it is not.

%% file: appendix.tex
\section{Appendix}
\label{sec:apx}

\subsection{Derivation of assignment-conditional densities} 
\label{apx:derivation}

In the main text, the anomaly-free density is represented by the fixed clean conditionals \( p(x^*_j|\pa^*_j)\). We use the same notation here.
The density under the normal reference distribution is induced by the causal factorization of the anomaly-free SCM over latent states: 

\begin{align}
  p_\text{normal}(\xb^\ast)   
  =
  \prod_{j=1}^d  p(x^*_j \mid \pa^*_j).
\end{align}
Since we assume perfect measurement in the anomaly-free setting, \(X_j=X_j^\ast\), the induced observed density is
\begin{align}
  p_\text{normal}(\xb)
  =
  \prod_{j=1}^d p(x_j \mid \pa_j).
\end{align}
If we now consider the density conditional on a certain intervention pattern $\ass=(\genb,\measb)$, we need to distinguish between mechanistic and measurement interventions. For simplicity, we first consider these cases separately. In case of a mechanistic intervention at node $j$, the latent value is replaced by a draw from an exogenous distribution, which is perfectly observed. Hence the conditional density is given by the typical causal factorization, except that the conditional for node $j$ is replaced by the exogenous distribution. We denote $\Genc_\ass=\{j:\gen_j=1\},\Mc_\ass=\{j:\meas_j=1\}$ and obtain

\begin{align}
  p(\xb\mid\boldsymbol z)   
  =
  \prod_{j \in \Genc_\ass} \tilde p_{\GenRVj}(x_j)
  \prod_{j \notin \Genc_\ass}  p(x_j \mid \pa_j).
\end{align}

For a measurement intervention, the observed value is replaced by a draw from an exogenous distribution, but the latent value that nonetheless propagated through the causal system is unobserved. Hence, we need to marginalize over this latent value. To separate the latent values that we integrate over from the observed values, we define \(\bar x_j\) to be the unobserved value before measurement replacement for \(j\in\Mc_\ass\), and set \(\bar x_j=x_j\) otherwise.
Let \(\bar{\pa}_j\) denote the parent vector evaluated at \(\bar{\xb}\). This means that the notation is somewhat overloaded since \(\bar x_j\) can both denote an integration variable, and an observed value, depending on whether \(j\in\Mc_\ass\) or not. We obtain
\begin{align}
  p(\xb \mid \boldsymbol w)
  =
  \prod_{j \in \Mc_\ass} \tilde p_{\MeasRVj}(x_j)
  \int
  \left[
  \prod_{j=1}^d  p(\bar x_j \mid \bar{\pa}_j)
  \right]
  \prod_{j \in \Mc_\ass} d\bar x_j .
  \label{eq:margin-meas-only}
\end{align}
By combining the two cases in the density under arbitrary assignments $\ass$, we obtain
\begin{align}
  p(\xb\mid \ass)
  =
  \prod_{j \in \Mc_\ass} \tilde p_{\MeasRVj}(x_j)
  \int
  \left[
  \prod_{j \notin \Genc_\ass} p(\bar x_j \mid \bar{\pa}_j)
  \prod_{j \in \Genc_\ass} \tilde p_{\GenRVj}(\bar x_j)
  \right]
  \prod_{j \in \Mc_\ass} d\bar x_j ,
  \label{eq:margin-derivation}
\end{align}
which is Eq.~\eqref{eq:margin} in the main text.

\subsection{Identifiability of intervention probabilities}
\label{apx:proofs}

We reduce identifiability of the intervention probabilities
\(\pib_{\rm red}\) to identifiability of the weights of fixed
assignment-conditional densities.

\begin{theorem}[Identifiability of reduced intervention probabilities]
Let $\{\ass_k\}_{k=1}^m$ be a subset of assignments whose
densities $p(\xb\mid \ass_k)$ correspond to singleton equivalence classes
under the reduced sink-node parameterization. Let $m_k$ denote their
mixture weights. Write the reduced Bernoulli coordinates as
\[
\qb
=
\big(
(\GenPj)_{j\in [d]\setminus\mathcal S},
(\MeasPj)_{j\in [d]\setminus\mathcal S},
(\rho_j)_{j\in\mathcal S}
\big),
\qquad
\rho_j := P(\gen_j=1\lor\meas_j=1).
\]
Assume:
\begin{enumerate}[noitemsep, topsep=0pt]
    \item[(i)] $q_\ell\in(0,1)$ for all reduced Bernoulli coordinates
    $q_\ell$ of $\qb$;
    \item[(ii)] the weights $m_k$ are uniquely identifiable from the full mixture;
    \item[(iii)] the matrix
    $A \in \mathbb{R}^{m \times (1+2d-|\mathcal S|)}$ with rows
    \[
        A_k = \bigl(1,\gen^{(k)}_{[d]\setminus\mathcal S},
                      \meas^{(k)}_{[d]\setminus\mathcal S},
                      s^{(k)}_{\mathcal S}\bigr)
    \]
    has full column rank.
\end{enumerate}
Then the reduced intervention probabilities $\pib_{\rm red}$ are
identifiable. In particular, at sink nodes only
$\rho_j=P(\gen_j=1\lor\meas_j=1)$ is identifiable, not $\GenPj$ and
$\MeasPj$ separately.
\label{thm:identifiability}
\end{theorem}

\begin{proof}
Let $\mathcal N := [d]\setminus \mathcal S$ denote the non-sink nodes and define
$s_j := z_j \lor w_j$ for $j\in\mathcal S$. For an assignment $\ass_k$, let
\[
y_k
=
\big(
\gen^{(k)}_{\mathcal N},
\meas^{(k)}_{\mathcal N},
s^{(k)}_{\mathcal S}
\big)
\in\{0,1\}^{2d-|\mathcal S|}.
\]
The reduced prior mass of $\ass_k$ can be written as
\[
m_k
=
\prod_{\ell}
q_\ell^{y_{k\ell}}(1-q_\ell)^{1-y_{k\ell}} .
\]
By assumption (ii), these weights $m_k$ are uniquely determined from the
observed mixture.

By assumption (i), all reduced Bernoulli coordinates satisfy
$q_\ell\in(0,1)$. Hence all selected weights are positive and logarithms are
well-defined. Taking logarithms and rewriting in log-odds form gives
\[
\log m_k
=
c+\sum_{\ell} y_{k\ell}\theta_\ell,
\]
where
\[
\theta_\ell := \log\frac{q_\ell}{1-q_\ell},
\qquad
c:=\sum_\ell \log(1-q_\ell).
\]
Equivalently, in the original notation,
\[
\theta
=
\big(c,(u_j)_{j\in\mathcal N},(v_j)_{j\in\mathcal N},(r_j)_{j\in\mathcal S}\big)^\top,
\]
with
\[
u_j := \log\frac{\GenPj}{1-\GenPj},
\qquad
v_j := \log\frac{\MeasPj}{1-\MeasPj},
\qquad
r_j := \log\frac{\rho_j}{1-\rho_j}.
\]
Stacking the equations gives
\[
A\theta=b,
\qquad
b=(\log m_1,\dots,\log m_m)^\top .
\]
By assumption (iii), $A$ has full column rank, so $\theta$ is uniquely
determined. Applying the inverse logit map to the non-intercept coordinates
recovers
\[
\GenPj,\MeasPj \quad (j\in\mathcal N),
\qquad
\rho_j \quad (j\in\mathcal S).
\]
Thus $\pib_{\rm red}$ is identifiable on the open parameter space.

Finally, the reduction at sink nodes identifies only
\[
\rho_j=P(\gen_j=1\lor\meas_j=1),
\]
because mechanistic and measurement interventions at a sink induce the same
reduced assignment-conditional distribution. Thus $\GenPj$ and $\MeasPj$ are
not separately identifiable at sink nodes. 
\end{proof}

\begin{remark}[Boundary cases]
The assumption $q_\ell\in(0,1)$ is used only to make the log-odds
parameterization well-defined. Some boundary cases can also be handled.

For a reduced Bernoulli coordinate $q_\ell$, the product form of the assignment
weights implies
\[
q_\ell=0
\quad\Rightarrow\quad
m_k=0 \text{ for every selected assignment with } y_{k\ell}=1,
\]
and
\[
q_\ell=1
\quad\Rightarrow\quad
m_k=0 \text{ for every selected assignment with } y_{k\ell}=0.
\]
Converses hold when the selected assignments contain the relevant opposing
indicator values and all other coordinates are compatible with positive mass.
Under such coverage conditions, boundary coordinates can be identified from the
zero pattern of the identifiable weights.

After fixing the identified boundary coordinates, one discards assignments whose
weights are forced to be zero and restricts the matrix $A$ to the columns
corresponding to the remaining free Bernoulli coordinates, together with the
intercept. On this restricted problem the remaining coordinates lie in $(0,1)$,
so the log-odds argument above applies. If the restricted assignment matrix has
full column rank, the remaining coordinates are uniquely identified.

In the anomaly setting of this paper, we assume that the all-clean assignment
has non-negligible, indeed dominant, mass. This excludes boundary cases with
$q_\ell=1$, since any such coordinate would assign zero probability to the
all-clean assignment.
\end{remark}

\begin{lemma}[Sufficient condition for identifiable selected weights]
For a set of fixed densities $\{ p_k \}_{1 \leq k \leq K}$, and a subset
$\{p_j\}_{j \in \mathcal{J}}$, assume that
\begin{enumerate}
    \item[(a)] the selected densities $\{p_j\}_{j\in\mathcal J}$ are linearly independent, and
    \item[(b)] their span intersects the span of the remaining densities only trivially:
    \[
    \mathrm{span}\{p_j:j \in \mathcal{J}\}
    \cap
    \mathrm{span}\{p_j:j \notin \mathcal{J}\}
    =
    \{0\}.
    \]
\end{enumerate}
Then the mixture weights $m_j$ of the selected components, $j\in\mathcal J$, are
uniquely identifiable from the full mixture.

A stronger but simpler sufficient condition is global linear independence of the
full density family $\{ p_k \}_{1 \leq k \leq K}$.
\label{lem:selected-weights}
\end{lemma}

\begin{proof}
Assume that two sets of mixture weights induce the same mixture density:
\[
\sum_{k=1}^K m_k p_k
=
\sum_{k=1}^K m'_k p_k .
\]
Then
\[
\sum_{k=1}^K (m_k-m'_k)p_k = 0.
\]
Splitting the sum into selected and remaining components gives
\[
\sum_{j\in\mathcal J} (m_j-m'_j)p_j
=
-
\sum_{k\notin\mathcal J} (m_k-m'_k)p_k .
\]
The left-hand side lies in
$\mathrm{span}\{p_j:j\in\mathcal J\}$, while the right-hand side lies in
$\mathrm{span}\{p_k:k\notin\mathcal J\}$. By assumption (b), their common value
must lie in the trivial intersection of these two spans, and hence must be zero:
\[
\sum_{j\in\mathcal J} (m_j-m'_j)p_j = 0.
\]
By assumption (a), the selected densities are linearly independent. Therefore
\[
m_j-m'_j=0
\qquad\text{for all } j\in\mathcal J.
\]
Thus the selected mixture weights are uniquely identifiable.

If the full family $\{p_k\}_{k=1}^K$ is linearly independent, then the same
argument applied to the full sum immediately yields $m_k=m'_k$ for every
$k=1,\dots,K$. In particular, the selected weights are identifiable.
\end{proof}

\subsection{Generic identifiability for linear Gaussian SCMs}

\label{apx:linear-gaussian}

We show in the following that for linear Gaussian SCMs, and Gaussian outlier distributions, the reduced intervention probabilities are generically identifiable.
\begin{theorem}[Generic identifiability in linear-Gaussian SCMs]
Fix a known DAG $G=(V,E)$ over $d$ variables with a known topological order.
Consider the linear-Gaussian SCM
\[
X_j^*=\sum_{k\in \mathrm{Pa}(j)} b_{kj}X_k^*+\varepsilon_j,
\qquad
\varepsilon_j\sim \mathcal N(0,\sigma_j^2),
\qquad \sigma_j^2>0,
\]
with mutually independent noises. Observations satisfy $X_j=X_j^*$ unless
a measurement intervention occurs.

For each node $j$, assume mechanistic and measurement replacement variables
share the same Gaussian distribution
\[
O_j^{\mathrm{mech}},O_j^{\mathrm{meas}}
\sim \mathcal N(\alpha_j,\tau_j^2),
\qquad \tau_j^2>0,
\]
independent of all other variables.

For each assignment $a=(z,w)$, let $p_a(x;\theta)$ denote the corresponding
assignment-conditional density, where
\[
\theta=(b_{kj},\sigma_j,\alpha_j,\tau_j)\in\Theta
\]
and the admissible parameter space is
\[
\Theta
=
(\mathbb R)^{|E|}
\times (0,\infty)^d
\times \mathbb R^d
\times (0,\infty)^d.
\]
Thus the coefficients \(b_{kj}\) and replacement means \(\alpha_j\) are unrestricted real
parameters, while standard deviations are strictly positive.

Define structural equivalence by
\[
a\sim a'
\quad\Longleftrightarrow\quad
p_a(x;\theta)=p_{a'}(x;\theta)
\quad\text{for all }\theta\in\Theta.
\]

Let $S$ denote the set of sink nodes of $G$.

Then, outside a Lebesgue-null subset of $\Theta$, the assignment densities
indexed by structural equivalence classes are linearly independent. Consequently,
the corresponding aggregate mixture weights are identifiable. Moreover, the clean
assignment together with singleton interventions yields a full-rank system, and
therefore the reduced intervention probabilities
\[
\qb
=
\big(
(\GenPj)_{j\in [d]\setminus\mathcal S},
(\MeasPj)_{j\in [d]\setminus\mathcal S},
(\rho_j)_{j\in\mathcal S}
\big),
\qquad
\rho_j := P(\gen_j=1\lor\meas_j=1)
\]
are generically identifiable for reduced Bernoulli coordinates satisfying \(q_\ell\in[0,1)\).
\label{thm:linear}
\end{theorem}
We first show in the following lemma, that for an arbitrary DAG, that the clean assignment and all singleton assignments, are pairwise not structurally equivalent, and also not structurally equivalent to any multi-intervention assignment. It is sufficient to show that we can construct one admissible parameter setting such that the assignment-conditional densities differ.
Together with the standard linear independence of finite families of distinct
Gaussian densities, this will imply that the selected densities are linearly
independent and have trivial intersection with the span of the remaining
components.

\begin{lemma}[Clean and singleton assignments are not merged with multi-intervention assignments]

Fix a known DAG $G$ with topological order $1,\dots,d$, and let $S$ denote
the set of sink nodes. Consider the linear-Gaussian SCM
\[
X_j^*
=
\sum_{k\in \mathrm{Pa}(j)} b_{kj}X_k^*
+
\varepsilon_j,
\qquad
\varepsilon_j\sim \mathcal N(0,\sigma_j^2),
\qquad
\sigma_j^2>0,
\]
with clean measurements $X_j=X_j^*$ unless a measurement intervention occurs.
For each node $j$, let mechanistic and measurement replacement variables have
the same distribution
\[
O_j^{\mathrm{mech}},O_j^{\mathrm{meas}}
\sim
\mathcal N(\alpha_j,\tau_j^2),
\qquad
\tau_j^2>0.
\]

Let $a_0$ denote the clean assignment. For $j\notin S$, let
$a_j^{\mathrm{mech}}$ denote the singleton mechanistic intervention at $j$, and
$a_j^{\mathrm{meas}}$ the singleton measurement intervention at $j$. For
$j\in S$, let $a_j^{\mathrm{sink}}$ denote the reduced singleton intervention
$s_j=z_j\vee w_j=1$.

There exists an admissible parameter setting $\theta_0$ such that the
assignment-conditional densities corresponding to
\[
a_0,\qquad
a_j^{\mathrm{mech}}\;(j\notin S),\qquad
a_j^{\mathrm{meas}}\;(j\notin S),\qquad
a_j^{\mathrm{sink}}\;(j\in S)
\]
are distinct, except for the intended sink-node collapse between $z_j$ and $w_j$ when $j\in S$, and are also distinct from any assignment involving two or more reduced interventions.

Hence, the selected assignments form distinct structural equivalence classes. 
\label{lem:singleton-not-merged}
\end{lemma}

\begin{proof}
We exhibit one admissible parameter setting at which the relevant
assignment-conditional distributions are distinct. This is sufficient to prove
non-structural-equivalence, since structural equivalence means equality of the
assignment-conditional densities for all admissible parameter values.

Choose
\[
\sigma_j^2=\tau_j^2=1
\qquad \text{for all } j.
\]
Choose every structural coefficient on an edge to be
\[
b_{kj}=1 .
\]
% This choice is admissible because all edge coefficients are allowed to be
% nonzero. 
Let
\[
\alpha_1,\dots,\alpha_d
\]
be algebraically independent over $\mathbb Q$. Such a choice is possible because
$\mathbb R$ contains finite sets of arbitrary size that are algebraically
independent over $\mathbb Q$. In particular, the $\alpha_j$ are linearly
independent over $\mathbb Q$, and each $\alpha_j$ is nonzero.

For any assignment $a=(z,w)$, write
\[
R(a):=\{j:z_j\vee w_j=1\}
\]
for the set of nodes with a reduced intervention.

We first record a consequence of the linear SCM under the above parameter
choice. For every assignment $a$ and every node $r$, the mean
$\mathbb E_a[X_r]$ is an integer linear combination of replacement means
\[
\{\alpha_\ell:\ell\preceq r\},
\]
where $\ell\preceq r$ means that $\ell=r$ or $\ell$ is an ancestor of $r$.
Indeed, this follows by induction along the topological order. If $w_r=1$, then
$X_r=O_r^{\mathrm{meas}}$, so
\[
\mathbb E_a[X_r]=\alpha_r.
\]
If $w_r=0$ and $z_r=1$, then $X_r=X_r^*=O_r^{\mathrm{mech}}$, so again
\[
\mathbb E_a[X_r]=\alpha_r.
\]
If $w_r=z_r=0$, then $X_r=X_r^*$ and
\[
\mathbb E_a[X_r]
=
\sum_{k\in \mathrm{Pa}(r)} \mathbb E_a[X_k^*],
\]
so the claim follows from the induction hypothesis for the parents of $r$.

Consequently, let $a$ and $a'$ be two assignments, and let $r$ be the earliest
node in topological order at which their reduced intervention indicators differ:
\[
z_r\vee w_r \neq z'_r\vee w'_r .
\]
Exactly one of the two assignments has a reduced intervention at $r$. The
assignment with the reduced intervention contributes the direct term
$\alpha_r$ to the mean of $X_r$, while the other assignment does not. Any
contribution to the mean of $X_r$ coming from strict ancestors of $r$ is an
integer linear combination of replacement means $\alpha_\ell$ with
$\ell\neq r$. Therefore
\[
\mathbb E_a[X_r]-\mathbb E_{a'}[X_r]
=
\pm \alpha_r+\sum_{\ell\neq r} c_\ell\alpha_\ell
\]
for some integers $c_\ell$. Since the $\alpha_\ell$ are linearly independent
over $\mathbb Q$, this difference is nonzero.

We now compare the selected assignments with arbitrary assignments.

\paragraph{Case 1: The clean assignment.}
Under the clean assignment $a_0$, all structural noises have mean zero and no
replacement variables are used. Hence
\[
\mathbb E_{a_0}[X_j]=0
\qquad \forall j.
\]
Let $a$ be any non-clean assignment. Then $R(a)\neq\emptyset$. Let $r$ be the
earliest node in $R(a)$. Since no strict predecessor of $r$ has a reduced
intervention, no replacement mean from an ancestor can propagate into $X_r$.
If $w_r=1$, then
\[
\mathbb E_a[X_r]=\alpha_r.
\]
If $w_r=0$ and $z_r=1$, then also
\[
\mathbb E_a[X_r]=\alpha_r.
\]
In both cases $\mathbb E_a[X_r]=\alpha_r\neq 0$, whereas
$\mathbb E_{a_0}[X_r]=0$. Therefore
\[
p_a(\cdot;\theta_0)\neq p_{a_0}(\cdot;\theta_0),
\]
and no non-clean assignment is structurally equivalent to the clean assignment.

\paragraph{Case 2: A selected singleton versus an assignment with a different reduced intervention set.}
Fix a selected singleton reduced assignment at node $j$. If $j\notin S$, this
selected singleton is either $a_j^{\mathrm{mech}}$ or
$a_j^{\mathrm{meas}}$. If $j\in S$, it is the reduced sink singleton
$a_j^{\mathrm{sink}}$, represented by one of the assignments with
$s_j=z_j\vee w_j=1$ and no other reduced interventions.

Let $a$ be any assignment such that
\[
R(a)\neq \{j\}.
\]
Let
\[
r
:=
\min_{\text{topological order}}
\bigl(R(a)\triangle \{j\}\bigr)
\]
be the earliest node at which the reduced intervention pattern of $a$ differs
from that of the selected singleton. By the preliminary argument,
\[
\mathbb E_a[X_r]-\mathbb E_{\mathrm{single}}[X_r]
=
\pm \alpha_r+\sum_{\ell\neq r} c_\ell\alpha_\ell
\]
for some integers $c_\ell$. The coefficient of $\alpha_r$ is nonzero. Hence,
by linear independence of $\alpha_1,\dots,\alpha_d$ over $\mathbb Q$,
\[
\mathbb E_a[X_r]\neq \mathbb E_{\mathrm{single}}[X_r].
\]
Thus the observed mean vectors differ, and therefore
\[
p_a(\cdot;\theta_0)\neq p_{\mathrm{single}}(\cdot;\theta_0).
\]
So $a$ is not structurally equivalent to the selected singleton.

This covers all assignments whose reduced intervention set is not exactly
$\{j\}$, including assignments that remove the singleton intervention at $j$,
assignments that add interventions at other nodes, and assignments that do both.

\paragraph{Case 3: A selected singleton versus assignments with the same reduced intervention set.}
It remains to compare assignments whose reduced intervention set is exactly
\[
R(a)=\{j\},
\]
but whose intervention type at $j$ may differ from the selected singleton.

\paragraph{Case 3a: Sink nodes.}
First suppose $j\in S$. Since $j$ is a sink node, it has no descendants in
$G[X^*]$. A mechanistic intervention at $j$ replaces $X_j^*$ by
$O_j^{\mathrm{mech}}$, and the observation satisfies $X_j=X_j^*$ unless a
measurement intervention occurs. A measurement intervention at $j$ replaces
$X_j$ by $O_j^{\mathrm{meas}}$. The two replacement variables have the same
distribution by assumption, and no downstream observed coordinates are affected
because $j$ has no descendants. Thus the assignments with
\[
(z_j,w_j)=(1,0),\quad (0,1),\quad (1,1)
\]
all induce the same observed distribution. This is exactly the intended
sink-node collapse represented by
\[
s_j=z_j\vee w_j.
\]

\paragraph{Case 3b: Non-sink nodes.}
Now suppose $j\notin S$. Then $j$ has at least one child in $G[X^*]$. Choose
\[
k\in \mathrm{Ch}(j).
\]
Because all edge coefficients have been set to $1$, the coefficient of $X_j^*$
in the linear expansion of $X_k^*$ is the number of directed paths from $j$ to
$k$. Since $k$ is a child of $j$, this number is at least one. Denote this
strictly positive integer by
\[
\beta_{jk}>0.
\]
Since no assignment considered in this case intervenes on $k$, we have
$X_k=X_k^*$.

We compare the three possible assignments with reduced intervention set
$\{j\}$.

\emph{Mechanistic singleton versus measurement singleton.}
Under the measurement singleton $a_j^{\mathrm{meas}}$, the latent variable
$X_j^*$ is not intervened on. Since all structural noises have mean zero and
there are no other interventions,
\[
\mathbb E_{a_j^{\mathrm{meas}}}[X_j^*]=0.
\]
Hence no mean shift is propagated from $j$ to $k$, and
\[
\mathbb E_{a_j^{\mathrm{meas}}}[X_k]=0.
\]
Under the mechanistic singleton $a_j^{\mathrm{mech}}$, the latent variable
$X_j^*$ is replaced by $O_j^{\mathrm{mech}}$ with mean $\alpha_j$. Therefore
the mean shift propagated to $k$ is
\[
\mathbb E_{a_j^{\mathrm{mech}}}[X_k]
=
\beta_{jk}\alpha_j.
\]
Since $\beta_{jk}>0$ and $\alpha_j\neq 0$, this mean is nonzero. Hence
$a_j^{\mathrm{mech}}$ and $a_j^{\mathrm{meas}}$ induce different observed
distributions and are not structurally equivalent.

\emph{Both interventions at $j$ versus the measurement singleton.}
Let $a_j^{\mathrm{both}}$ denote the assignment with
\[
z_j=w_j=1
\]
and no other reduced interventions. Under $a_j^{\mathrm{both}}$, the latent
variable $X_j^*$ is replaced by $O_j^{\mathrm{mech}}$, so the same mean shift
is propagated to $k$:
\[
\mathbb E_{a_j^{\mathrm{both}}}[X_k]
=
\beta_{jk}\alpha_j.
\]
Under $a_j^{\mathrm{meas}}$,
\[
\mathbb E_{a_j^{\mathrm{meas}}}[X_k]=0.
\]
Thus $a_j^{\mathrm{both}}$ and $a_j^{\mathrm{meas}}$ induce different observed
distributions and are not structurally equivalent.

\emph{Both interventions at $j$ versus the mechanistic singleton.}
Under the mechanistic singleton $a_j^{\mathrm{mech}}$, the observed coordinate
satisfies
\[
X_j=X_j^*=O_j^{\mathrm{mech}}.
\]
In the linear expansion of $X_k$, the coefficient multiplying $X_j^*$ is
$\beta_{jk}$, while all remaining terms are independent of
$O_j^{\mathrm{mech}}$. Hence
\[
\operatorname{Cov}_{a_j^{\mathrm{mech}}}(X_j,X_k)
=
\operatorname{Cov}(O_j^{\mathrm{mech}},\beta_{jk}O_j^{\mathrm{mech}})
=
\beta_{jk}\tau_j^2
\neq 0.
\]
Under $a_j^{\mathrm{both}}$, the observed coordinate is
\[
X_j=O_j^{\mathrm{meas}},
\]
whereas the latent variable propagating to descendants is
\[
X_j^*=O_j^{\mathrm{mech}}.
\]
The variables $O_j^{\mathrm{meas}}$ and $O_j^{\mathrm{mech}}$ are independent,
and all other terms in $X_k$ are independent of $O_j^{\mathrm{meas}}$. Hence
\[
\operatorname{Cov}_{a_j^{\mathrm{both}}}(X_j,X_k)=0.
\]
Thus $a_j^{\mathrm{both}}$ and $a_j^{\mathrm{mech}}$ induce different covariance
matrices and are not structurally equivalent.

Therefore, for every non-sink node $j$, the mechanistic singleton, the
measurement singleton, and the assignment with both interventions at $j$ are
pairwise not structurally equivalent.

Combining Cases 1--3, the clean assignment and the selected singleton reduced
assignments are not structurally equivalent to any other assignment, except for
the intended sink-node collapse at original sink nodes. This collapse is exactly
represented by the reduced intervention
\[
s_j=z_j\vee w_j .
\]
\end{proof}

Now we show that one such parameter setting $\theta_0$ is sufficient to establish that the linear independence conditions hold generically, in the sense that the set of parameters for which they do not hold is a Lebesgue-null subset of the parameter space.
\begin{proof}[Proof of Theorem~\ref{thm:linear}]
We proceed in steps.

\paragraph{Step 1: Gaussian form.}
Under any assignment $a$, the SCM remains linear Gaussian. Since both marginalization, and hard replacement preserve the Gaussian property, we have
\[
p_a(x;\theta)=\mathcal N(x;\mu_a(\theta),\Sigma_a(\theta)).
\]
Both $\mu_a(\theta)$ and $\Sigma_a(\theta)$ are polynomial functions of $\theta$.

\paragraph{Step 2: Generic separation.}
For $a,a'$, define
\[
F_{a,a'}(\theta)
=
(\mu_a(\theta)-\mu_{a'}(\theta),
\mathrm{vec}(\Sigma_a(\theta)-\Sigma_{a'}(\theta))).
\]
Each coordinate is polynomial. If $a\not\sim a'$, then there exists
$\theta_0$ such that $F_{a,a'}(\theta_0)\neq 0$, hence at least one coordinate
polynomial is nonzero. Its zero set has measure zero.

Taking a finite union over all assignment pairs gives a measure-zero set $E$
such that, for all $\theta\notin E$, distinct structural equivalence classes
yield distinct Gaussian distributions. Classical finite-mixture identifiability
results imply that distinct non-degenerate Gaussian component densities are
linearly independent \citep[][]{yakowitz:1968:identifiability}. Hence,
generically, the assignment-conditional densities indexed by structural
equivalence classes are linearly independent outside a Lebesgue-null subset of
parameter space.

\paragraph{Step 3: Full-rank construction.}
By Lemma~\ref{lem:singleton-not-merged}, the clean and the reduced singleton assignments form distinct equivalence classes. Their reduced assignment rows are
\[
(1,0,\dots,0)
\]
and standard basis perturbations in the coordinates
\[
(z_{[d]\setminus S},w_{[d]\setminus S},s_S).
\]
Subtracting the clean row yields the standard basis of $\mathbb R^{2d-|S|}$,
so the matrix has rank $1+2d-|S|$.

\paragraph{Step 4: Recovery of probabilities.}
If \(q_\ell\in(0,1)\) for all reduced Bernoulli coordinates, all conditions of
Theorem~\ref{thm:identifiability} are generically satisfied. We exclude the
case \(q_\ell=1\), since it would assign zero probability to the all-clean
assignment, contrary to the rare-anomaly setting. Zero-boundary cases
\(q_\ell=0\) can nevertheless be handled in the generic construction: the
selected assignments include the all-clean assignment and all singleton reduced
interventions. Let \(m_0\) be the all-clean weight and \(m_\ell\) the singleton
weight for reduced coordinate \(\ell\). Since
\[
m_0=\prod_r(1-q_r),\qquad
m_\ell=q_\ell\prod_{r\neq \ell}(1-q_r),
\]
we have, whenever \(m_0>0\),
\[
q_\ell=\frac{m_\ell}{m_0+m_\ell}.
\]
Thus \(q_\ell=0\) is identified exactly when the corresponding singleton
assignment has zero weight.

\end{proof}

\subsection{Generic linear independence beyond linear-Gaussian SCMs}
\label{apx:polynomial-gaussian}

We now extend the generic linear-independence result from linear-Gaussian SCMs
to fixed-degree polynomial Gaussian SCMs.  The argument uses two facts.  First,
the polynomial Gaussian model and its linear-Gaussian submodel induce the same
structural equivalence classes over assignments.  Second, for polynomial
Gaussian SCMs, moments of polynomial test functions are polynomial, hence
real analytic, functions of the model parameters.

Throughout this subsection, fix a DAG \(G=(V,E)\) and a maximum polynomial
degree \(r\ge 1\).  For each node \(j\), write
\[
f_j(x_{\mathrm{Pa}(j)};\beta_j)
=
\sum_{\gamma\in\Gamma_j(r)}
\beta_{j,\gamma}x_{\mathrm{Pa}(j)}^\gamma ,
\]
where \(\Gamma_j(r)\) is the finite set of multi-indices of total degree at
most \(r\) over the parent variables of \(j\).  Let
\[
\beta=(\beta_{j,\gamma})_{j,\gamma}
\]
collect all polynomial coefficients, and let
\[
\theta=(\beta,\sigma,\alpha,\tau),
\]
where
\[
\sigma=(\sigma_1,\ldots,\sigma_d)\in(0,\infty)^d,\qquad
\alpha=(\alpha_1,\ldots,\alpha_d)\in\mathbb R^d,\qquad
\tau=(\tau_1,\ldots,\tau_d)\in(0,\infty)^d.
\]
Thus the natural parameter space is
\[
\Theta
=
\mathbb R^{|\Gamma|}
\times (0,\infty)^d
\times \mathbb R^d
\times (0,\infty)^d,
\qquad
|\Gamma|:=\sum_{j=1}^d|\Gamma_j(r)|.
\]

The clean structural equations are
\[
X_j^\ast
=
f_j(X^\ast_{\mathrm{Pa}(j)};\beta_j)
+
\sigma_j\eta_j,
\qquad
\eta_j\sim\mathcal N(0,1),
\]
and the clean measurement equations are deterministic identities,
\[
X_j=X_j^\ast .
\]
A mechanistic intervention at \(j\) replaces the structural equation for
\(X_j^\ast\) by
\[
X_j^\ast
=
O_j^{\mathrm{mech}}
=
\alpha_j+\tau_j\eta_j^{\mathrm{mech}},
\]
whereas a measurement intervention at \(j\) replaces the measurement equation by
\[
X_j
=
O_j^{\mathrm{meas}}
=
\alpha_j+\tau_j\eta_j^{\mathrm{meas}}.
\]
All clean noises and replacement noises are mutually independent, and
\[
\eta_j^{\mathrm{mech}},\eta_j^{\mathrm{meas}}
\sim \mathcal N(0,1)
\]
are independent copies.  Thus, for each node \(j\),
\(O_j^{\mathrm{mech}}\) and \(O_j^{\mathrm{meas}}\) have the same
node-specific Gaussian distribution but are independent draws.  Replacement
distributions are node-specific: replacement sources at different nodes are not
identified with one another, even if some parameter values make their Gaussian
distributions accidentally coincide.

For the polynomial Gaussian class, write
\[
a\sim_{\mathrm{poly}}a'
\]
if the assignment-conditional observed laws under \(a\) and \(a'\) are equal
for every admissible polynomial-Gaussian parameter value
\(\theta\in\Theta\).  Similarly, write
\[
a\sim_{\mathrm{lin}}a'
\]
if the assignment-conditional observed laws are equal for every admissible
linear-Gaussian parameter value.  These are structural equivalence relations:
they require equality for all admissible parameter values in the corresponding
model class, not merely equality at a particular parameter value.
\begin{definition}[Observed active graph]
Fix an assignment \(a=(z,w)\).  The active graph associated with \(a\) is
constructed from the clean causal graph as follows.

The graph contains latent vertices \(X_1^\ast,\ldots,X_d^\ast\), observed
vertices \(X_1,\ldots,X_d\), and active source vertices
\[
C_j \quad (z_j=0),\qquad
R_j^{\mathrm{mech}} \quad (z_j=1),\qquad
R_j^{\mathrm{meas}} \quad (w_j=1).
\]
The source labels are
\[
\lambda(C_j)=(\mathrm{clean},j),
\qquad
\lambda(R_j^{\mathrm{mech}})
=
\lambda(R_j^{\mathrm{meas}})
=
(\mathrm{repl},j).
\]
The equality of the two replacement labels reflects that
\(O_j^{\mathrm{mech}}\) and \(O_j^{\mathrm{meas}}\) are independent copies from
the same node-specific replacement distribution.

The directed edges are:
\begin{itemize}[noitemsep, topsep=0pt]
    \item \(C_j\to X_j^\ast\) if \(z_j=0\), and
    \(R_j^{\mathrm{mech}}\to X_j^\ast\) if \(z_j=1\);
    \item \(X_k^\ast\to X_j^\ast\) for every \(k\in\mathrm{Pa}(j)\), but only
    if \(z_j=0\);
    \item \(X_j^\ast\to X_j\) if \(w_j=0\), and
    \(R_j^{\mathrm{meas}}\to X_j\) if \(w_j=1\).
\end{itemize}

Thus, a mechanistic intervention removes the incoming latent causal edges into
\(X_j^\ast\), whereas a measurement intervention removes only the measurement
edge \(X_j^\ast\to X_j\).  In particular, when \(z_j=w_j=1\), the mechanistic
replacement source may still propagate through \(X_j^\ast\) to latent
descendants, while the measurement replacement source affects only the observed
coordinate \(X_j\).

We then reduce the graph in two steps.
\begin{enumerate}[noitemsep, topsep=0pt]
    \item Remove every non-observed vertex that has no directed path to any
    observed vertex, together with its incident edges.
    \item If \(R_j^{\mathrm{mech}}\) reaches observations only through the
    deterministic identity measurement path
    \[
    R_j^{\mathrm{mech}}\to X_j^\ast\to X_j,
    \]
    contract this path into a single edge
    \[
    R_j^{\mathrm{mech}}\to X_j .
    \]
\end{enumerate}
The second reduction can occur only when \(w_j=0\), since otherwise the
measurement edge \(X_j^\ast\to X_j\) is absent.  If \(w_j=1\), the mechanistic
replacement source is either pruned, if it has no observed descendants, or
retained, if it propagates to downstream observed variables.

The resulting reduced graph is called the observed active graph and is denoted
\(\bar H(a)\).

Two assignments \(a\) and \(a'\) are graphically equivalent, written
\(a\equiv_{\mathrm{gr}}a'\), if \(\bar H(a)\) and \(\bar H(a')\) are
isomorphic by a directed graph isomorphism that fixes every latent vertex
\(X_j^\ast\), fixes every observed vertex \(X_j\), and preserves source labels
\(\lambda\).  Source vertices with the same label may be permuted.
\end{definition}
\begin{lemma}[Polynomial and linear models have the same structural equivalences]
\label{lem:poly-linear-same-equivalences}
Assume the setup above, and assume \(r\ge1\), so that the polynomial Gaussian
SCM class contains the linear-Gaussian submodel.  Then, for any two assignments
\(a,a'\),
\[
a\sim_{\mathrm{poly}}a'
\quad\Longleftrightarrow\quad
a\equiv_{\mathrm{gr}}a'
\quad\Longleftrightarrow\quad
a\sim_{\mathrm{lin}}a' .
\]
Hence the polynomial Gaussian class and the linear-Gaussian submodel induce the
same structural equivalence classes over assignments.
\end{lemma}

\begin{proof}
We first prove that graphical equivalence implies structural equivalence.
Suppose \(a\equiv_{\mathrm{gr}}a'\).  The observed active graph contains exactly
the exogenous source variables, active structural mechanisms, and active
measurement maps that can affect the observed vector.  The isomorphism between
\(\bar H(a)\) and \(\bar H(a')\) fixes every latent vertex \(X_j^\ast\) and
every observed vertex \(X_j\).  Hence it preserves which named structural
mechanism \(f_j\) is evaluated at each latent node and which observed coordinate
is produced at each measurement node.  The isomorphism may only permute source
vertices with the same label, and such vertices are independent copies with the
same distribution.

The pruning step is distribution-preserving because pruned vertices have no
directed path to any observed variable.  The contraction step is
distribution-preserving because it contracts only the deterministic identity
measurement path
\[
R_j^{\mathrm{mech}}\to X_j^\ast\to X_j,
\qquad X_j=X_j^\ast,
\]
and only when \(R_j^{\mathrm{mech}}\) has no other observed descendants.  A
contracted mechanistic replacement source at node \(j\) has the same label as a
measurement replacement source at node \(j\) because these are independent
copies from the same node-specific replacement distribution.

Therefore \(a\) and \(a'\) induce the same recursive computation of the
observed vector in distribution for every admissible polynomial-Gaussian
parameter value.  Hence
\[
a\sim_{\mathrm{poly}}a'.
\]
Since the linear-Gaussian model is a submodel of the polynomial Gaussian model,
the same argument gives
\[
a\sim_{\mathrm{lin}}a'.
\]

It remains to prove the converse.  Suppose
\[
a\not\equiv_{\mathrm{gr}}a'.
\]
We show that \(a\) and \(a'\) are already separated inside the
linear-Gaussian submodel.

Assign a distinct formal indeterminate \(b_e\) to every latent edge
\(e\in E(G[X^\ast])\).  Let
\[
\mathcal R=\mathbb R[b_e:e\in E(G[X^\ast])]
\]
be the corresponding polynomial ring, and let \(\mathcal K\) be its fraction
field.

For each active source vertex \(s\) in \(\bar H(a)\), define its transfer vector
\[
T_a(s)\in \mathcal R^d
\]
by
\[
[T_a(s)]_i
=
\sum_{\pi:s\leadsto X_i}
\prod_{e\in \pi\cap E(G[X^\ast])} b_e ,
\]
where the sum is over all directed paths from \(s\) to the observed vertex
\(X_i\) in \(\bar H(a)\).  Source edges and clean measurement edges have
coefficient one.  Define the labeled transfer multiset
\[
\mathcal T(a)
=
\big\{(\lambda(s),T_a(s)):
s\text{ is an active source of }\bar H(a)\big\},
\]
with multiplicity.  Define \(\mathcal T(a')\) analogously.

We claim that \(\mathcal T(a)\) determines \(\bar H(a)\) up to
label-preserving isomorphism.  Indeed, each monomial in the \(i\)-th coordinate
of \(T_a(s)\) corresponds to a directed source-to-observation path from \(s\) to
\(X_i\).  Since the latent graph is acyclic and all latent vertices are named, a
directed path is uniquely determined by its latent edge set together with its
source and observed target.  Direct source-to-observation paths correspond to
the empty monomial in the appropriate coordinate.  All coefficients are
nonnegative integers, so path contributions cannot cancel.  Hence the support
of the transfer polynomials recovers exactly the labeled
source-to-observation path set.  By construction of \(\bar H(a)\), every
retained non-observed vertex and every retained edge lies on at least one
directed path from an active source to an observed vertex.  Therefore the union
of the recovered paths is precisely the reduced observed active graph.  Thus
\[
\mathcal T(a)=\mathcal T(a')
\quad\Longrightarrow\quad
a\equiv_{\mathrm{gr}}a' .
\]
By contraposition,
\[
a\not\equiv_{\mathrm{gr}}a'
\quad\Longrightarrow\quad
\mathcal T(a)\neq\mathcal T(a')
\]
as labeled multisets.

Now work symbolically in the linear-Gaussian submodel with zero intercepts.
Let the clean source \(C_j\) have mean \(0\) and variance \(\sigma_j^2\), and
let every replacement source with label \((\mathrm{repl},j)\) have mean
\(\alpha_j\) and variance \(\tau_j^2\).  All sources are mutually independent.
Treat
\[
\{b_e\}_{e\in E(G[X^\ast])},
\qquad
\{\alpha_j,\sigma_j^2,\tau_j^2\}_{j=1}^d
\]
as formal indeterminates.

Under assignment \(a\), the symbolic observed mean and covariance are
\[
\mu_a
=
\sum_{j=1}^d
\alpha_j
\sum_{s:\lambda(s)=(\mathrm{repl},j)}
T_a(s),
\]
and
\[
\Sigma_a
=
\sum_{j=1}^d
\sigma_j^2
\sum_{s:\lambda(s)=(\mathrm{clean},j)}
T_a(s)T_a(s)^\top
+
\sum_{j=1}^d
\tau_j^2
\sum_{s:\lambda(s)=(\mathrm{repl},j)}
T_a(s)T_a(s)^\top .
\]
The same formulas hold for \(a'\).

Assume, for contradiction, that
\[
\mu_a=\mu_{a'},
\qquad
\Sigma_a=\Sigma_{a'}.
\]
Since \(\alpha_j,\sigma_j^2,\tau_j^2\) are formal indeterminates, equality of
these polynomial expressions implies equality of the coefficient of each such
indeterminate.

For a clean label \((\mathrm{clean},j)\), there is at most one active clean
source.  Let its transfer vector under \(a\) be \(t_j(a)\), and set
\(t_j(a)=0\) if the source is absent after pruning.  Define \(t_j(a')\)
analogously.  Equality of the coefficient of \(\sigma_j^2\) gives
\[
t_j(a)t_j(a)^\top
=
t_j(a')t_j(a')^\top .
\]
Over the fraction field \(\mathcal K\), equality of two rank-one outer products
implies equality of the vectors up to sign:
\[
t_j(a')=\pm t_j(a).
\]
Because transfer vectors have entries in \(\mathcal R\) with nonnegative
coefficients, the negative sign is possible only when both vectors are zero.
Thus
\[
t_j(a)=t_j(a').
\]

For a replacement label \((\mathrm{repl},j)\), there are at most two active
replacement sources, corresponding to the mechanistic and measurement
replacement at node \(j\).  Let their transfer vectors under \(a\) be
\(u_j(a)\) and \(v_j(a)\), using the zero vector for an absent or pruned source.
Define \(u_j(a')\) and \(v_j(a')\) analogously.  Equality of the coefficient of
\(\alpha_j\) in the mean gives
\[
u_j(a)+v_j(a)=u_j(a')+v_j(a'),
\]
and equality of the coefficient of \(\tau_j^2\) in the covariance gives
\[
u_j(a)u_j(a)^\top+v_j(a)v_j(a)^\top
=
u_j(a')u_j(a')^\top+v_j(a')v_j(a')^\top .
\]
Combining the two identities yields
\[
\big(u_j(a)-v_j(a)\big)\big(u_j(a)-v_j(a)\big)^\top
=
\big(u_j(a')-v_j(a')\big)\big(u_j(a')-v_j(a')\big)^\top .
\]
Again over \(\mathcal K\), this implies
\[
u_j(a')-v_j(a')
=
\pm\big(u_j(a)-v_j(a)\big).
\]
Together with equality of the sums, this gives
\[
\{u_j(a),v_j(a)\}
=
\{u_j(a'),v_j(a')\}
\]
as unordered pairs.

Therefore equality of the symbolic means and covariances would force equality
of every clean-source transfer vector and every replacement-source transfer
multiset.  Equivalently,
\[
\mathcal T(a)=\mathcal T(a'),
\]
contradicting \(a\not\equiv_{\mathrm{gr}}a'\).  Hence at least one entry of
\[
\mu_a-\mu_{a'}
\quad\text{or}\quad
\Sigma_a-\Sigma_{a'}
\]
is a nonzero polynomial in the symbolic linear-Gaussian parameters.

The admissible linear-Gaussian parameter space contains a nonempty open set:
edge coefficients \(b_e\in\mathbb R\), replacement means
\(\alpha_j\in\mathbb R\), and positive variances
\(\sigma_j^2,\tau_j^2>0\).  A nonzero polynomial cannot vanish on this whole
open set.  Therefore there exists an admissible linear-Gaussian parameter value
with positive variances such that
\[
(\mu_a,\Sigma_a)\neq(\mu_{a'},\Sigma_{a'}).
\]
At this parameter value, the observed Gaussian laws differ:
\[
p_a(\cdot;\theta)\neq p_{a'}(\cdot;\theta).
\]
Thus
\[
a\not\sim_{\mathrm{lin}}a'.
\]
Since the linear-Gaussian model is contained in the polynomial Gaussian model,
also
\[
a\not\sim_{\mathrm{poly}}a'.
\]

Combining the two directions proves
\[
a\sim_{\mathrm{poly}}a'
\quad\Longleftrightarrow\quad
a\equiv_{\mathrm{gr}}a'
\quad\Longleftrightarrow\quad
a\sim_{\mathrm{lin}}a',
\]
as claimed.
\end{proof}

Now, we can use this to extend the result of linear independence of equivalence classes under the linear model to the polynomial model

\begin{lemma}[Finite witness extension]
  \label{lem:finite-witness-extension}
Let \(\Theta\subseteq\mathbb R^q\) be connected and open, and let
\(\mathcal A=\{a_1,\ldots,a_m\}\) be finite. For each
\(a_i\in\mathcal A\) and \(\theta\in\Theta\), let
\(\mu_i^\theta\) be a probability measure on \(\mathbb R^d\).

Suppose there exist measurable test functions
\(h_1,\ldots,h_m:\mathbb R^d\to\mathbb R\) such that
\[
M_{i\ell}(\theta)
:=
\int h_\ell(x)\,d\mu_i^\theta(x)
\]
is well-defined and real analytic in \(\theta\) on \(\Theta\). If there exists
\(\theta_0\in\Theta\) such that
\[
D(\theta_0)
:=
\det\left[M_{i\ell}(\theta_0)\right]_{i,\ell=1}^m
\neq 0,
\]
then
\[
\mu_1^\theta,\ldots,\mu_m^\theta
\]
are linearly independent for Lebesgue-almost-every \(\theta\in\Theta\).
\end{lemma}

\begin{proof}
Define
\[
D(\theta)
:=
\det\left[M_{i\ell}(\theta)\right]_{i,\ell=1}^m .
\]
By assumption, \(D\) is real analytic on \(\Theta\). Since
\(D(\theta_0)\neq0\), \(D\) is not identically zero. The zero set of a
nontrivial real-analytic function on a connected open set has Lebesgue measure
zero. Hence \(D(\theta)\neq0\) for Lebesgue-almost-every \(\theta\).

For such \(\theta\), suppose that the measures were linearly dependent. Then
there would exist \(c=(c_1,\ldots,c_m)\neq0\) such that
\[
\sum_{i=1}^m c_i\mu_i^\theta=0
\]
as a finite signed measure. Integrating against \(h_\ell\), for
\(\ell=1,\ldots,m\), gives
\[
\sum_{i=1}^m c_iM_{i\ell}(\theta)=0.
\]
Equivalently, \(c^\top M(\theta)=0\). Since \(D(\theta)\neq0\), the matrix
\(M(\theta)\) is invertible, and therefore \(c=0\), a contradiction.
\end{proof}

\begin{lemma}[Polynomial moments under interventions]
  \label{lem:polynomial-moments}
Consider a known-DAG polynomial Gaussian additive-noise SCM of fixed degree at
most \(r\). For each node,
\[
X_j^\ast
=
f_j(X^\ast_{\mathrm{Pa}(j)};\beta_j)+\sigma_j\eta_j,
\qquad
\eta_j\sim\mathcal N(0,1),
\]
where \(f_j\) is polynomial and \(\sigma_j>0\). Let mechanistic and measurement
replacement variables be independent Gaussian variables with the same
node-specific distribution
\[
O_j^{\mathrm{mech}}
=
\alpha_j+\tau_j\eta_j^{\mathrm{mech}},
\qquad
O_j^{\mathrm{meas}}
=
\alpha_j+\tau_j\eta_j^{\mathrm{meas}},
\qquad
\tau_j>0,
\]
where
\[
\eta_j^{\mathrm{mech}},\eta_j^{\mathrm{meas}}
\sim \mathcal N(0,1)
\]
are independent standard Gaussian variables.

Let \(\theta\) collect the structural coefficients, noise standard deviations,
replacement means, and replacement standard deviations:
\[
\theta=(\beta,\sigma,\alpha,\tau).
\]
For a fixed assignment \(a\), let
\[
\mu_a^\theta=\mathcal L_\theta^a(X)
\]
denote the observed assignment-conditional distribution. If
\(h:\mathbb R^d\to\mathbb R\) is polynomial, then
\[
\theta\mapsto
\int h(x)\,d\mu_a^\theta(x)
\]
is polynomial in \(\theta\).
\end{lemma}

\begin{proof}
Collect all independent standard Gaussian variables into
\[
E=
(\eta_1,\ldots,\eta_d,
\eta_1^{\mathrm{mech}},\ldots,\eta_d^{\mathrm{mech}},
\eta_1^{\mathrm{meas}},\ldots,\eta_d^{\mathrm{meas}}).
\]
For an assignment \(a=(z,w)\), the variables are generated recursively in a
topological order by
\[
X_j^\ast
=
\begin{cases}
f_j(X^\ast_{\mathrm{Pa}(j)};\beta_j)+\sigma_j\eta_j,
& z_j=0,\\[2mm]
\alpha_j+\tau_j\eta_j^{\mathrm{mech}},
& z_j=1,
\end{cases}
\]
and
\[
X_j
=
\begin{cases}
X_j^\ast,
& w_j=0,\\[2mm]
\alpha_j+\tau_j\eta_j^{\mathrm{meas}},
& w_j=1.
\end{cases}
\]
Here \(\alpha_j\) and \(\tau_j\) are shared replacement parameters for the
mechanistic and measurement replacement distributions at node \(j\), while
\(\eta_j^{\mathrm{mech}}\) and \(\eta_j^{\mathrm{meas}}\) are independent
standard Gaussian variables.

Since the graph is acyclic and all \(f_j\) are polynomial, recursive
substitution yields
\[
X=T_a(E;\theta),
\]
where every coordinate of \(T_a\) is polynomial in \(E\) and in the model
parameters
\[
\theta=(\beta,\sigma,\alpha,\tau).
\]

Thus, for polynomial \(h\), the composition
\[
h(T_a(E;\theta))
\]
is polynomial in \(E\) and \(\theta\). Taking expectation with respect to the
standard Gaussian vector \(E\) replaces monomials in \(E\) by finite Gaussian
moments, so
\[
\mathbb E[h(T_a(E;\theta))]
\]
is polynomial in \(\theta\).
\end{proof}

\begin{lemma}[Polynomial tests separate finite Gaussian mixtures]
Let
\[
\nu=\sum_{i=1}^m c_i\,\mathcal N(\mu_i,\Sigma_i)
\]
be a finite signed combination of non-degenerate Gaussian probability measures
on \(\mathbb R^d\). If
\[
\int h(x)\,d\nu(x)=0
\]
for every polynomial \(h\), then \(\nu=0\).
\label{lem:polynomial-tests}
\end{lemma}

\begin{proof}
The moment-generating function of \(\nu\) is
\[
M_\nu(t)
=
\int e^{t^\top x}\,d\nu(x)
=
\sum_{i=1}^m c_i
\exp\left(
t^\top\mu_i+\frac12t^\top\Sigma_i t
\right).
\]
This is an entire function of \(t\in\mathbb C^d\). Its derivatives at \(t=0\)
are the polynomial moments of \(\nu\), which vanish by assumption. Hence
\(M_\nu\equiv0\). Therefore the characteristic function
\[
\widehat\nu(s)=M_\nu(is)
\]
vanishes for all \(s\in\mathbb R^d\). By uniqueness of the Fourier transform for
finite signed measures, \(\nu=0\).
\end{proof}

\begin{lemma}[Finite polynomial witnesses at a Gaussian point]
Let
\[
\mu_1,\ldots,\mu_m
\]
be linearly independent non-degenerate Gaussian probability measures on
\(\mathbb R^d\). Then there exist polynomial test functions
\[
h_1,\ldots,h_m
\]
such that
\[
\det\left[
\int h_\ell(x)\,d\mu_i(x)
\right]_{i,\ell=1}^m
\neq 0 .
\]
\label{lem:finite-polynomial-witnesses}
\end{lemma}

\begin{proof}
Suppose no such finite collection of polynomial tests existed. Then the
polynomial moment functionals would have rank strictly smaller than \(m\) on
\(\mathrm{span}\{\mu_1,\ldots,\mu_m\}\). Hence there would exist coefficients
\(c_1,\ldots,c_m\), not all zero, such that
\[
\nu=\sum_{i=1}^m c_i\mu_i
\]
annihilates every polynomial test function. By lemma~\ref{lem:polynomial-tests} \(\nu=0\),
contradicting the assumed linear independence of
\(\mu_1,\ldots,\mu_m\). Therefore a finite collection of polynomial tests with
full-rank moment matrix exists.
\end{proof}

\begin{theorem}[Generic identifiability for polynomial Gaussian SCMs]
Fix a known DAG \(G\), and consider the polynomial Gaussian additive-noise SCM
class of fixed degree at most \(r \ge 1\) with parameter space \(\Theta\) defined
above. For each \(\theta\in\Theta\), let
\(\{\mu_a^\theta:a\in\mathcal A\}\) denote the assignment-conditional observed
probability measures indexed by the full finite set of reduced assignment
classes, where assignments are quotiented by structural equivalences, including
the sink-node collapse \(s_j=z_j\vee w_j\) for \(j\in S\).

Then, outside a Lebesgue-null subset of \(\Theta\), the assignment-conditional
observed measures \(\{\mu_a^\theta:a\in\mathcal A\}\) are linearly independent
as finite signed measures. Moreover, the corresponding assignment-conditional
densities are linearly independent, the aggregate mixture weights of all reduced
assignment classes are identifiable, and the reduced intervention probabilities
\[
\{\pi_j^{\mathrm{mech}},\pi_j^{\mathrm{meas}}:j\notin S\},
\qquad
\{\rho_j=P(z_j\vee w_j):j\in S\}
\]
are identifiable from the fixed-component mixture, for all reduced Bernoulli
coordinates \(q_\ell\in[0,1)\).
\end{theorem}

\begin{proof}
\textbf{Generic linear independence of measures.}

By Lemma~\ref{lem:poly-linear-same-equivalences}, the reduced assignment
classes in the polynomial Gaussian class coincide with those in the
linear-Gaussian submodel. By Theorem~\ref{thm:linear}, we can therefore choose a
witness point \(\theta_0\) in the linear-Gaussian submodel, with all
higher-order polynomial coefficients set to zero, outside the corresponding
Lebesgue-null exceptional set. At this point, the assignment-conditional
densities indexed by the full reduced structural equivalence classes are
linearly independent. Since these
densities are densities with respect to Lebesgue measure, the corresponding
observed probability measures
\[
\mu_{a_1}^{\theta_0},\ldots,\mu_{a_m}^{\theta_0}
\]
are linearly independent as finite signed measures. Moreover, at \(\theta_0\)
the model is linear Gaussian, so these measures are non-degenerate Gaussian
measures. Lemma~\ref{lem:finite-polynomial-witnesses} therefore implies that
there exist polynomial test functions \(h_1,\ldots,h_m\) such that
\[
D(\theta_0)
:=
\det\left[
\int h_\ell(x)\,d\mu_{a_i}^{\theta_0}(x)
\right]_{i,\ell=1}^m
\neq0 .
\]

For general \(\theta\in\Theta\), define
\[
D(\theta)
:=
\det\left[
\int h_\ell(x)\,d\mu_{a_i}^{\theta}(x)
\right]_{i,\ell=1}^m .
\]
By Lemma~\ref{lem:polynomial-moments}, each matrix entry is polynomial in
\(\theta=(\beta,\sigma,\alpha,\tau)\). Hence \(D\) is a polynomial, and
therefore real analytic, function on \(\Theta\). Since \(D(\theta_0)\neq0\),
\(D\) is not identically zero. Because \(\Theta\) is connected and open, the
zero set of \(D\) has Lebesgue measure zero. For every \(\theta\) outside this
zero set, Lemma~\ref{lem:finite-witness-extension} implies that
\[
\mu_{a_1}^{\theta},\ldots,\mu_{a_m}^{\theta}
\]
are linearly independent as finite signed measures. 

\textbf{From measures to densities and mixture weights.}
For every \(\theta\in\Theta\) and every assignment \(a\), the observed law
\(\mu_a^\theta\) is absolutely continuous with respect to Lebesgue measure.
Indeed, all structural noises and replacement variables have strictly positive
Gaussian densities and enter additively with positive standard deviations.
For each observed coordinate \(X_j\), choose an independent Gaussian source that
enters \(X_j\) directly: \(\eta_j\) if \(z_j=w_j=0\),
\(\eta_j^{\mathrm{mech}}\) if \(z_j=1,w_j=0\), and
\(\eta_j^{\mathrm{meas}}\) if \(w_j=1\). Conditional on all remaining
exogenous sources, the map from these \(d\) chosen Gaussian variables to
\((X_1,\ldots,X_d)\) is triangular in topological order with nonzero diagonal
entries \(\sigma_j\) or \(\tau_j\). Therefore the conditional law has a
Lebesgue density, and integrating over the remaining exogenous variables
preserves absolute continuity. Thus we may write
\[
d\mu_a^\theta(x)=p_a^\theta(x)\,dx .
\]
Any linear relation among the measures has Radon--Nikodym derivative
equal to the corresponding linear relation among the densities. Hence linear
independence of the measures and of the densities coincide. Therefore, outside
the same null set, the reduced assignment-conditional densities are linearly
independent. By Lemma~\ref{lem:selected-weights}, the aggregate mixture weights
of all reduced assignment classes are identifiable in the associated
fixed-component mixture.

\textbf{Recovery of reduced intervention probabilities.}

The full assignment set contains the clean assignment and all singleton
interventions. By Lemma~\ref{lem:singleton-not-merged}, each of these
assignments is not structurally equivalent to any other assignment, except for
the intended sink-node collapse: the lemma exhibits a linear-Gaussian parameter
setting at which the corresponding assignment-conditional densities differ.
Because this linear-Gaussian setting is contained in the polynomial model class,
the same non-equivalence holds for the polynomial class. Thus the full reduced
assignment set contains distinct classes for the clean assignment and all
reduced singleton interventions, whose aggregate weights are identifiable.

For these selected classes, the reduced assignment matrix \(A\) has one row
\[
(1,0,\ldots,0)
\]
for the clean assignment and one row
\[
(1,e_\ell)
\]
for each reduced singleton coordinate
\[
(z_{[d]\setminus S},w_{[d]\setminus S},s_S).
\]
Subtracting the clean row from each singleton row gives the standard basis of
\(\mathbb R^{2d-|S|}\). Hence \(A\) has full column rank.

If all reduced Bernoulli coordinates satisfy \(q_\ell\in(0,1)\), the log-odds
argument of Theorem~\ref{thm:identifiability} recovers the reduced intervention
probabilities. The case \(q_\ell=1\) is excluded by the rare-anomaly setting,
since it would assign zero probability to the all-clean assignment. Zero
boundary cases \(q_\ell=0\) are also identifiable from the clean and singleton
weights. Let \(m_0\) be the all-clean weight and \(m_\ell\) the singleton weight
for coordinate \(\ell\). Since
\[
m_0=\prod_r(1-q_r),
\qquad
m_\ell=q_\ell\prod_{r\neq \ell}(1-q_r),
\]
we have, whenever \(m_0>0\),
\[
q_\ell=\frac{m_\ell}{m_0+m_\ell}.
\]
Thus \(q_\ell=0\) is identified exactly when the corresponding singleton weight
is zero. This recovers
\[
\pi_j^{\mathrm{mech}},\quad \pi_j^{\mathrm{meas}}\quad (j\notin S),
\qquad
\rho_j \quad (j\in S).
\]
\end{proof}

\begin{remark}[Scope]
The results above cover two analytically tractable model classes: linear
Gaussian SCMs and fixed-degree polynomial Gaussian additive-noise SCMs containing
the linear-Gaussian submodel. The finite-witness strategy may extend to broader
finite-dimensional nonlinear SCM classes whenever the relevant test moments are
real analytic in the model parameters, but establishing such moment analyticity
requires additional assumptions, for example on domains and integrability.

Our identifiability statement is deliberately limited. We do not claim to recover
the structural parameters of the clean SCM or the parameters of the replacement
distributions. Rather, we show that, for generic fixed choices of these
parameters, the reduced intervention probabilities are identifiable from the
corresponding fixed-component mixture.

This is separate from the identifiability of the clean SCM itself. In nonlinear
additive-noise models, the causal direction and, more generally, the DAG can be
identifiable under standard ANM assumptions, except for known degenerate cases
such as the linear-Gaussian setting \citep{hoyer:2008:nonlinear}. In linear Gaussian SCMs, the DAG is
generally identifiable only up to Markov equivalence without additional
assumptions; full DAG identifiability can be obtained under equal error
variances \citep{peters:2014:identifiability}. 
\end{remark}

\subsection{Algorithmic Details}
\label{apx:algorithm}
\begin{algorithm}
\caption{\textsc{\ourmethod}}
\label{alg:anomdetect}
\begin{algorithmic}
\STATE \textbf{Input:} DAG $\mathcal{G}$, dataset $\Xb$
\STATE \textbf{Output:} anomaly assignment $\mathbf{A}=(\mathbf{\Gen},\mathbf{\Meas})$,
feature scores $\mathcal{S}$, confidences $\mathcal{C}$

\FORALL{$ X_k \in \mathcal{G}$}
  \STATE $\hat{f}_k \leftarrow \textsc{RobustlyRegress}(\mathrm{Pa}_k, X_k)$
  \STATE $R_k \leftarrow X_k - \hat{f}_k(\Pa_k)$
  \STATE $\mathcal{L}(R_k) \leftarrow \textsc{EstimateNoise}(R_k)$
\ENDFOR

\STATE $\mathbf{A} \leftarrow \textsc{MLE-Assign}(\Xb,\mathcal{G},\{\hat f_k\},\{\mathcal{L}(R_k)\})$
\STATE $\mathcal{S} \leftarrow \textsc{Score}(\mathbf{A},\Xb,\mathcal{G},\{\hat f_k\},\{\mathcal{L}(R_k)\})$ %(Eq.~\eqref{eq:scores})
\STATE $\mathcal{C} \leftarrow \textsc{Conf}(\mathbf{A},\Xb,\mathcal{G},\{\hat f_k\},\{\mathcal{L}(R_k)\}) $ %(Eq.~\eqref{eq:conf})

\STATE \textbf{return} $\mathbf{A}, \mathcal{S}, \mathcal{C}$
\end{algorithmic}
\end{algorithm}
This section provides additional details on the proposed algorithm. Algorithm~\ref{alg:anomdetect} in Sec.~\ref{sec:algo} illustrates the high-level procedure. Here, we elaborate on the key components: the fitting of the causal mechanisms, the estimation of the noise distributions, and the heuristic assignment and gain calculations.

One practical approximation is that we do not compute the full scores in Equation~\eqref{eq:scores} for all samples and features. For samples that are unassigned and thus likely clean, we only compute the $O(1)$ mechanistic gains, avoiding the costly marginalization over all features for measurement gains. The rationale is that unassigned samples are expected to have low anomaly scores, so considering both outlier types provides minimal additional benefit.  

\subsubsection{Causal Mechanisms}
To estimate the causal mechanisms, we use \texttt{qgam} \citep{fasiolo:2021qgam} with $q = 0.5$. We select the number of knots per parent as a simple rule of thumb:
\[
m_{j,k} = \min\left\{\max\Big(5, \lfloor N^{1/3}/2 \rfloor \Big), 30 \right\}.
\]
Knots are placed at empirical quantile breakpoints of the parent covariates, so that dense regions receive more knots (allowing flexibility) while sparse regions receive fewer (reducing variance).  

\subsubsection{Marginalization}
\label{apx:marginalization}

Here, we provide further details on how to approximate the integral
\begin{align}
\int 
\underbrace{\prod_{j \notin \Genc_\ass}  p(x_j \mid \pa_j)}_{\text{'clean' likelihood}}
\underbrace{\prod_{j \in \Genc_\ass} \tilde p_{\GenRV_j}(x_j)}_{\text{mechanistic outlier density}}
\underbrace{\prod_{j \in \Mc_\ass} \mathrm{d}x_j}_{\text{latent clean values}} .
\label{eq:margin_app}
\end{align}

We approximate \eqref{eq:margin_app} via Monte Carlo integration over the latent clean values.
Let $\ass = (\genb,\measb)$ be a fixed assignment and define
$\Genc_\ass = \{j : \gen_j = 1\}$ and $\Mc_\ass = \{j : \meas_j = 1\}$.
For each $j \in \Mc_\ass$, let $U_j$ denote the structural noise variable of node $j$.
Since we explicitly learn the distributions $P_{U_j}$, which capture the clean conditional noise, from the residuals of the causal mechanism fits, we can sample from these distributions.
Latent clean values are obtained by sampling $M$ times from the corresponding conditional distributions $P_{U_j}$ for nodes in $\Mc_\ass$ and propagating samples in topological order, and the resulting likelihoods are then averaged:

\begin{align}
\int 
\underbrace{\prod_{j \notin \Genc_\ass}  p(x_j \mid \pa_j)}_{\text{'clean' likelihood}}
\underbrace{\prod_{j \in \Genc_\ass} \tilde p_{\GenRV_j}(x_j)}_{\text{mechanistic outlier density}}
\underbrace{\prod_{j \in \Mc_\ass} \mathrm{d}x_j}_{\text{latent clean values}} \\
\approx\frac{1}{M} \sum_{m=1}^M 
\underbrace{\prod_{j \notin \Genc_\ass \cup \Mc_\ass}  p(x_j^{(m)} \mid \pa_j^{(m)})}_{\text{'clean' likelihood}}
\underbrace{\prod_{j \in \Genc_\ass \setminus \Mc_\ass} \tilde p_{\GenRV_j}(x_j^{(m)})}_{\text{mechanistic outlier density}},
\end{align}

\[
x_j^{(m)} =
\begin{cases}
\hat f_j(\pa_j^{(m)}) + \tilde u_j^{(m)}, & j \in \Mc_\ass \setminus \Genc_\ass,\\
\tilde o_j^{(m)}, & j \in \Mc_\ass \cap \Genc_\ass,\\
x_j, & j \notin \Mc_\ass,
\end{cases}
\qquad
\tilde u_j^{(m)} \sim P_{U_j},\tilde o_j^{(m)} \sim P_{\GenRVj}.
\]

By the law of large numbers, the Monte Carlo approximation converges to the true marginalization for $M \to \infty$.

\subsubsection{Heuristic MLE Assignment}  
The assignment procedure is summarized in Algorithm~\ref{alg:mleassign}. The algorithm iteratively and greedily accepts candidate outlier assignments based on \emph{optimistic gains}, which can be estimated without performing full marginalization. The procedure is decomposed into two main steps: assigning measurement outliers and then attributing residuals to mechanistic outliers.  

Measurement outliers are first considered by sorting candidates according to their optimistic gains. Each candidate’s gain is compared against the optimal alternative gain obtained from mechanistic attribution. Only candidates with positive optimistic gains that exceed the mechanistic alternative are accepted. The process continues until no remaining measurement candidates have positive optimistic gains.  

For a measurement outlier candidate at node $j$ and sample $i$, the optimistic gain $g^\mathit{opt}_{i,j}$ assumes that all affected residual likelihoods improve maximally, and then subtracts the cost of encoding the outlier and updating the Bernoulli probability. Specifically, the Bernoulli probability update can be approximated using the log of the binomial coefficient: encoding a sequence of $k$ successes and $n-k$ failures with prefix codes based on empirical probabilities yields  
\[
- k \log \frac{k}{n} - (n-k) \log \frac{n-k}{n} \approx \log \binom{n}{k},
\]  
i.e., the sequence-level Bernoulli code length is asymptotically equivalent to the binomial code length. For node $j$ with $\sum_{i=1}^N \Meas_j$ current measurement outliers, adding a new sample incurs a prior cost  
\[
- \log \tilde p_{\MeasRVj}  + \Big( \log \binom{N}{\sum_{i=1}^N \Meas_j + 1} - \log \binom{N}{\sum_{i=1}^N \Meas_j} \Big).
\]  

The full optimistic measurement gain is then
\begin{align*}
g^\mathit{meas-opt}_{i,j} = & \sum_{k \in \{j\} \cup \mathrm{Ch}(j)} 
\Big[- \log \mathcal{L}(r_{ik}) - \min_r \{- \log \mathcal{L}(r) \} \Big]  + \log \tilde p_{\MeasRVj}(x_{i,j})  \\
& - \Big[  \log \binom{N}{\sum_{i=1}^N \Meas_j + 1} - \log \binom{N}{\sum_{i=1}^N \Meas_j} \Big],
\end{align*}  
where $\mathrm{Ch}(j)$ denotes the children of node $j$, $\mathcal{L}(r_{ik})$ is the likelihood of the residual for node $k$ and sample $i$, and $\tilde p_{\MeasRVj}$ is the current estimated prior of a measurement outlier at node $j$.  

Intuitively,
\begin{enumerate}
    \item The first term captures the maximal improvement in likelihood for affected nodes.  
    \item The second term penalizes the prior probability of adding a new outlier.  
    \item The third term accounts for the combinatorial cost of updating the number of outliers.  
\end{enumerate}

Mechanistic gains are defined analogously, but depend only on the residual of the node itself, i.e.
\begin{align*}
g^\mathit{mech-opt}_{i,j} = & 
- \log \mathcal{L}(r_{ij}) - \min_r \{- \log \mathcal{L}(r) \}  \\
& + \log \tilde p_{\GenRVj}(x_{i,j}) \\
& - \Big[  \log \binom{N}{\sum_{i=1}^N \Gen_j + 1} - \log \binom{N}{\sum_{i=1}^N \Gen_j} \Big].
\end{align*}  
If a node is not in the MB of a measurement outlier, this estimate is exact. Otherwise, exact gains are evaluated by calculating the likelihood after adding the outlier to the assignment, approximating marginalization as needed.

The overall procedure is as follows: \textsc{MLE-Assign} performs the greedy measurement and mechanistic assignment passes. Candidates are processed in descending order of optimistic gain for measurements, and descending residual cost for mechanistic outliers. The final assignment $ \mathbf{A} = (\mathbf{\Gen}, \mathbf{\Meas})$ is then used to compute scores and classes.

\begin{algorithm}[tb]
\caption{\textsc{MLE-Assign}}
\label{alg:mleassign}
\begin{algorithmic}
\STATE \textbf{Input:} dataset $D$ of size $N$, DAG $\mathcal{G}$,
regressors $\{\hat f_k\}$, likelihoods $\{\mathcal{L}(R_k)\}$
\STATE \textbf{Output:} assignment $\mathbf{A}=(\mathbf{\Gen},\mathbf{\Meas})$

\STATE \textbf{Notation:}
$g^{\text{meas-opt}}_{i,j}/g^{\text{mech-opt}}_{i,j}$ optimistic gains;
$g^{\text{meas}}_{i,j}/g^{\text{mech}}_{i,j}$ actual gains

\STATE \textbf{Initialization:}
\STATE $\mathbf{A} \leftarrow (\mathbf{0},\mathbf{0})$
\STATE Compute $g^{\text{opt}}_{i,j}$ for all $(i,j)$
\STATE $C \leftarrow \text{sorted}(\{(i,j)\}, g^{\text{meas-opt}}_{i,j})$ (descending)

\STATE \textbf{Measurement pass:}
\FORALL{$(i,j) \in C$ \textbf{with} $g^{\text{opt}}_{i,j} > 0$}
  \STATE Compute $g^{\text{meas}}_{i,j}$
  \STATE $g^{\text{comp}}_{i,j} \leftarrow 
    \sum_{k \in \{j\} \cup \mathrm{Ch}(j),\; g^{\text{mech}}_{i,k} > 0}
    g^{\text{mech}}_{i,k}$
  \IF{$g^{\text{meas}}_{i,j} > g^{\text{comp}}_{i,j}$}
    \STATE $\mathbf{\Meas}[i,j] \leftarrow 1$
    \STATE Update $g^{\text{opt}}_{i,\ell}$ for affected $(i,\ell)$
  \ENDIF
\ENDFOR

\STATE \textbf{Mechanistic pass:}
\STATE $C \leftarrow \text{sorted}(\{(i,j)\}, g^{\text{mech-opt}}_{i,j})$ (descending)
\FORALL{$(i,j) \in C$}
  \IF{$\mathbf{\Meas}[i,j] = 1$}
    \STATE \textbf{continue}
  \ENDIF
  \STATE Compute $g^{\text{mech}}_{i,j}$
  \IF{$g^{\text{mech}}_{i,j} > 0$}
    \STATE $\mathbf{\Gen}[i,j] \leftarrow 1$
  \ELSE
    \STATE \textbf{continue}
  \ENDIF
\ENDFOR

\STATE \textbf{return} $\mathbf{A}=(\mathbf{\Gen},\mathbf{\Meas})$
\end{algorithmic}
\end{algorithm}

\subsubsection{Complexity.}
Fitting one \texttt{qgam} per non-root node \(k\) has complexity
\(O(Nm_k^2)\), where \(m_k\) is the number of spline basis functions used for
node \(k\).  With
\[
m_k = O\!\left(N^{1/3}|\Pa_k|\right),
\]
the total fitting cost is
\[
\sum_{k=1}^d O(Nm_k^2)
=
O\!\left(
N^{5/3}\sum_{k=1}^d |\Pa_k|^2
\right).
\]
In the dense worst case this is \(O(N^{5/3}d^3)\), while for bounded in-degree
\(\Delta\) it is \(O(N^{5/3}d\Delta^2)\).

Mechanistic gains are inexpensive once the clean residual likelihoods are
cached, giving \(O(1)\) per candidate.  For measurement gains, we perform a
constant number of model evaluations in the Markov-blanket space \(\MB_j\).
Thus computing measurement gains for all samples and features costs
\[
O\!\left(
N\sum_{j=1}^d |\MB_j|
\right),
\]
which is \(O(Nd^2)\) in the dense worst case and \(O(Nd\Delta_{\MB})\) when
Markov blankets have size at most \(\Delta_{\MB}\).

The greedy search evaluates only a linear number of candidate gains per sample,
rather than enumerating the exponential assignment space.  Therefore the total
worst-case complexity is
\[
O\!\left(
N^{5/3}\sum_{k=1}^d |\Pa_k|^2
+
N\sum_{j=1}^d |\MB_j|
\right).
\]
In the dense worst case this becomes
\[
O(N^{5/3}d^3 + Nd^2)
=
O(N^{5/3}d^3),
\]
whereas for bounded in-degree and bounded Markov-blanket size it is
\[
O(N^{5/3}d + Nd).
\]
In practice, parent sets and Markov blankets are typically much smaller than
\(d\), and only a subset of measurement candidates is evaluated.

\subsection{Implementation Details}
\label{apx:implementation}

In this section, we provide both the implementation details of all algorithms, we compared to, and the hyperparameter selection that we performed.

\subsubsection{Implementation}

\ourmethod{} is implemented in Python, and uses the R package \texttt{qgam} \citep{fasiolo:2021qgam}. Most of our competing detection methods, namely \AUTOENCODER \citep{sakurada:2014:aead}, \DEEPSVDD \citep{ruff:2018:deepsvdd}, \IFOREST \citep{liu:2008:isolation}, \COPOD \citep{li:2020:copod}, \LOF \citep{breunig:2000:lof}, are implemented with the Python library \texttt{pyod} \citep{zhao:2019:pyod}. We implemented the \NORMALFLOW-based \citep{jimenez:2015:normalizingflows} anomaly detection using the Python library \texttt{pzflow} \citep{crenshaw:2024:pzflow}. 

The competing RCA methods are combined with the best sample detector (aside from \ourmethod) of the respective experiment. For the synthetic data, this is \AUTOENCODER and for the real data \NORMALFLOW. After applying the sample detectors, we threshold such that the highest $p \%$ of samples with the highest scores are removed from training, where $p$ corresponds to the  true percentage of anomalies. 

For the calculation of Shapley values, we use the kernel explainer implemented in the \texttt{shap} library \citep{lundberg:2017:shap} with 100 background samples. For \SMOOTHTRAVERSAL and \SCOREORDER, we use the published implementations. For \DICE \citep{mothilal:2020:dice}, and \MOE \citep{angiulli:2024:m2oe}, we adapted the implementations published by the authors to our setting. We combined \DICE with a thresholded autoencoder to simulate a classification problem.

As mentioned in the main paper, \SCM is an adaptation of existing SCM-based attribution approaches \citep{budhathoki:2022:rootcause,strobl:2023:samplerca,strobl:2023:hetero}.
The key distinction in our setting is that we aim to attribute anomalousness at the level of an entire sample, rather than the outlierness of an individual feature \citep{budhathoki:2022:rootcause} or a binary target \citep{strobl:2023:samplerca}. In our formulation, sample anomalousness can be quantified with the joint density, which factorizes into independent noise distributions. Consequently, the conditional likelihoods already capture the individual feature contributions, making an explicit attribution step unnecessary.
For the SCM, we use the exact same implementation as \ourmethod to ensure a fair comparison.

\subsubsection{Hyperparameters}
Since all of our tasks are unsupervised, we tuned all hyperparameters, by selecting the ones with the best performance in our synthetic setup with different seeds than the data used for the actual evaluation.

\paragraph{\ourmethod}
The only hyperparameter of \ourmethod{} are the number of Monte Carlo samples, which we fixed to $K=100$ across all experiments, and the KDE parameters. We fix the trimmed fraction $\alpha=0.01$, and use a Gaussian kernel with bandwidth selected via Silverman's rule of thumb~\citep{silverman:2018:density} across all experiments.

\paragraph{Detection methods}
For competing detection methods, we conducted grid searches over:
$
\mathit{bs} \in \{32, 64, 128\}, \quad 
e \in \{10, 100, 500\}, \quad 
\mathit{est} \in \{100, 200, 300\}, \quad
\mathit{neigh} \in \{10, 20, 40\}.
$
The best configurations are summarized in Tab.~\ref{tab:sample-hparams}.

\begin{table}[h]
\centering
\caption{Best hyperparameters for detection methods.}
\label{tab:sample-hparams}
\resizebox{0.5\textwidth}{!}{%
  \input{tables/hyperpar_detection.tex}
}
\end{table}

\paragraph{RCA methods}
For the RCA methods, we searched over the same ranges of epochs and batch sizes. For \MOE, we additionally tuned the number of neighbors in the KNN step, $K \in \{10, 50, 100\}$, and for \SMOOTHTRAVERSAL and \SCOREORDER we tested the anomaly detectors \{\textsc{MedianCDFQuantile},\textsc{InverseDensity}\}. The selected configurations are shown in Tab.~\ref{tab:feature-hparams}.

\begin{table}[h]
\centering
\caption{Best hyperparameters for feature-level baselines.}
\label{tab:feature-hparams}
\resizebox{0.5\textwidth}{!}{%
  \input{tables/hyperpar_rca.tex}
}
\end{table}

\paragraph{\MARGINAL}
For \MARGINAL, we also compare two marginal outlier detectors based on estimated density and median distance, as well as significance levels $\alpha \in \{0.001,0.01,0.05,0.1\}$ with and without Bonferroni correction. The best configuration uses the density-based detector with $\alpha=0.01$ with Bonferroni correction.

\subsubsection{Error bars}
Unless stated otherwise, shaded regions in all plots denote mean \(\pm 1.96\) standard errors across 20 independent runs.

\subsubsection{Hardware}
All experiments were run on a CPU cluster. Each individual run used a single CPU core, and no GPU acceleration was required. Peak memory usage was below 8 GB per run. The reported experiments required approximately 200 CPU-hours in total, dominated by repeated synthetic and robustness runs. Representative individual runtimes for the main methods are reported in Appx.~\ref{apx:runtime}.

\subsection{Datasets}
\label{apx:datasets}

This section provides further detail about the synthetic and real-world datasets.
\subsubsection{Synthetic data}
We generate synthetic observations $X \in \mathbb{R}^{N=2000 \times D=15}$ from a structural causal model (SCM) over a random DAG $\mathcal{G}$ with $15$ nodes. Each edge $(i,j)$, $i>j$, is included independently with probability $p_e = 0.3$, yielding an acyclic graph.

Clean variables follow
\[
X_j^* = f_j(\mathrm{Pa}_j) + U_j,
\]
where $U_j \sim \mathcal{N}(0, \sigma^2)$ with $\sigma = 0.1$, and $f_j$ is an additive polynomial mechanism normalized to $[0,1]$ on the clean data.

\textbf{Polynomial generation:} Each function $f_j$ is constructed as a sum of univariate polynomial components over its parent nodes. For a single parent input $A$, we generate a polynomial of length $l$ as
\[
f(A) = \sum_{i=1}^l a_i (A - b_i)^{e_i},
\]
where
\begin{itemize}
    \item $a_i$ are random coefficients sampled uniformly from $[-1,-0.1] \cup [0.1,1]$,
    \item $b_i$ are random shifts sampled uniformly from $[-3,3]$,
    \item $e_i \in \{2,3\}$ are randomly chosen exponents.
\end{itemize}

For multivariate parents, the node function is the sum of the univariate polynomials applied independently to each parent:
\[
f_j(\mathrm{Pa}_j) = \sum_{k \in \mathrm{Pa}_j} f_{jk}(X_k),
\]
where each $f_{jk}$ is generated as above. After evaluating the raw polynomials on the parent data, the outputs are normalized to $[0,1]$ using min/max scaling to ensure numerical stability. We set $l=2$ for all experiments.

Asymmetric outliers are introduced independently by sampling from the true distribution of $X_j^*$ and shifting by $k$ empirical standard deviations. For mechanistic outliers, downstream nodes are recalculated with resampled noise. This procedure is repeated across five seeds, and averages are reported.

We choose outlier ratios to obtain 10\% sample contamination, with measurement and mechanistic outliers equally probable. Specifically,
\[
\MeasPj = \GenPj = 1 - (1 - 0.1)^{\frac{1}{2\cdot 15}} \approx 0.0035.
\]

In every synthetic experiment, we average results over 20 random seeds, generating new graphs, mechanisms, and outlier injections each time.
  
\subsubsection{Sachs dataset}

The Sachs dataset consists of simultaneous measurements of 11 phosphorylated proteins and phospholipids in human immune cells \citep{sachs:2005:causal}. The data comprise both observational and interventional samples collected under multiple experimental conditions, and are commonly used as a benchmark for causal discovery. A consensus causal graph, derived from prior biological knowledge, is typically treated as ground truth. We present this graph in Fig.~\ref{fig:sachs_dag}. Although the underlying biological signaling network may contain feedback loops, the benchmark graph used in most evaluations is represented as a directed acyclic graph (DAG). In our experiments, we follow this convention and remove the edge from \texttt{PIP3} to \texttt{plcg} to ensure acyclicity. The dataset contains $853$ observational samples, and we introduce outliers as follows.

We simulate two types of outliers: (i)~mechanistic, created by injecting a small fraction of interventional samples into the observational set, and (ii)~measurement, created by selecting a column of an observational sample and replacing its value with the one of a value from a sample of the respective interventional distribution.
For each variable, we introduce five mechanistic and five measurement outliers. We repeat this procedure across 20 random seeds.

\subsubsection{Causal Chambers dataset}

In addition, we use the \emph{Causal Chamber} benchmark \citep{gamella:2025:causalchamber}. It provides observational and interventional data from a controlled physical system with known ground-truth DAG. The  graph is provided in Fig.~\ref{fig:causal_chamber}. In each run we sample $2\,000$ observational samples and introduce outliers as follows.

As for the Sachs dataset, we simulate two types of outliers: (i)~mechanistic, created by injecting a small fraction of interventional samples into the observational set, and (ii)~measurement, created by selecting a column of an observational sample and replacing its value with the one of a value from a sample of the respective interventional distribution. Precisely, we randomly add $\frac{N}{20} \frac{0.05}{0.95} \approx$ 26 random samples from each interventional dataset to contaminate approximately 5\% of samples with mechanistic outliers. Additionally, we create the same amount of measurement outliers for each variable, by randomly selecting a column, and replacing it with the interventional column of an interventional sample. Again, we repeat this procedure across 20 random seeds.

The benchmark includes interventions labels as \emph{mid} and \emph{strong} for all variables but $L_{32},L_{31},L_{22},L_{21},L_{12},L_{11}$. We evaluate the performance on the different strengths separately. Since the Sachs dataset contains much weaker interventions, we can separate the datasets by weak, medium and strong interventions. We provide the distribution of absolute z-scores of the interventional samples in Fig.~\ref{fig:sachs_intervention_strengths}.

\begin{figure}
  \begin{subfigure}{0.33\textwidth}
    \centering
    \input{figs/sachs_shifts.tex}
    \caption{Sachs}
  \end{subfigure}%
  \begin{subfigure}{0.33\textwidth}
    \centering
    \input{figs/cc_mid_shifts.tex}
    \caption{Causal Chamber, mid}
  \end{subfigure}%
  \begin{subfigure}{0.33\textwidth}
    \centering
    \input{figs/cc_strong_shifts.tex}
    \caption{Causal Chamber, strong }
  \end{subfigure}%
  
  \caption{Distribution of absolute z-scores of interventional samples in the Sachs dataset (left), and Causal Chamber dataset with mid (center) and strong (right) interventions.}
  \label{fig:sachs_intervention_strengths}
\end{figure}

\begin{figure}
  \input{figs/sachs_dag.tex}
  \centering
  \caption{Consensus DAG of the Sachs dataset.}
  \label{fig:sachs_dag}
\end{figure}

\begin{figure}
  \input{figs/causal_chamber.tex}
  \centering
  \caption{DAG of the light tunnel dataset of the \emph{Causal Chamber} benchmark.}
  \label{fig:causal_chamber}
\end{figure}

\subsubsection{NYC Yellow Taxi Trip Dataset}

\begin{figure}[t]
    \centering
    \input{figs/case_study_dag.tex}
    \caption{Causal relationships of the NYC Taxi Dataset.}
    \label{fig:case_study_dag}
\end{figure}%
For the case study, we use the New York City Taxi \& Limousine Commission (NYC TLC) Yellow Taxi trip dataset obtained from Kaggle \citep{elemento:2015:nyc}. The analysis is restricted to data from January 2015 only. The dataset contains trip-level records for yellow taxis in New York City, including pickup and drop-off timestamps, trip distance, fare-related variables, passenger count, and geographic pickup and drop-off information recorded as coordinates. As an administrative dataset collected by the NYC TLC, it provides comprehensive coverage of yellow taxi activity during the study period. We preprocess the data by removing records with trivially missing or incorrect values such as zero coordinates, and select a random sample of size of $5\,000$ for our case study.

\subsubsection{Dataset licenses and terms.}
The Sachs protein-signaling dataset is publicly available; we use the version
distributed on Zenodo under a CC BY 4.0 license. The Causal Chambers datasets
are distributed by the authors under a CC BY 4.0 license. The NYC Yellow Taxi
Trip records are publicly available from the NYC Taxi and Limousine Commission
(TLC); we use them according to the applicable NYC TLC / NYC Open Data terms of
use.

\subsection{Additional Related Work}
\label{apx:related_work}
In addition, to the related work in the field of root cause analysis, we briefly discuss the broader field of outlier detection and explainable outlier detection.
\paragraph{Outlier Detection}
Outlier detection aims to identify observations that deviate from the bulk of the data. Traditional statistical methods, such as feature-wise Z-scores and Mahalanobis distance \citep{mahalanobis:1936:mahalanobi}, are interpretable but limited in capturing complex dependencies. Machine learning-based approaches, including Local Outlier Factor \citep[LOF][]{breunig:2000:lof}, Isolation Forests \citep{liu:2008:isolation}, and One-Class SVM \citep{schoelkopf:01:svdd}, improve flexibility and robustness to structured deviations. Deep models, such as autoencoders \citep{sakurada:2014:aead}, normalizing flows \citep{jimenez:2015:normalizingflows}, and VAEs \citep{kingma:2013:vae}, further capture highly nonlinear relationships at the cost of interpretability. For a comprehensive survey, see \citep{chandola:2009:anomalysurvey,pang:2021:deepanomalysurvey}.

\paragraph{Explainable Outlier Detection}
To improve interpretability, feature attribution methods like \LIME~\citep{ribeiro:2016:lime} and \SHAP~\citep{lundberg:2017:shap} have been adapted to anomaly scores \citep{friedman:2021:anomshap}, while counterfactual approaches \citep{wachter:2017:counterfactual,mothilal:2020:dice} identify minimal changes that would alter predictions. More recent methods \citep{angiulli:2024:m2oe} use masked transformation-based optimization to localize contributing features. For a survey on the field, see \citep{li:2023:survey}.
These methods provide feature-level explanations, but are largely correlation-based and rely on the learned structure of non-interpretable black-box models, which makes correct interpretation and classification of outlier types challenging.

\subsection{Additional Experiments}
In this section, we present a series of additional experiments designed to further evaluate and analyze our approach.

\label{apx:experiments}

\subsubsection{Additional metrics}
\label{apx:metrics}
In this section, we report additional evaluation metrics for the root cause localization experiments presented in the main paper.

In Tab.~\ref{tab:cc-causal-discovery-diagnostics}, we report the causal discovery diagnostics for the Causal Chambers dataset. We report DAG Precision, Recall, F1 and SHD for both the AE-filtered and clean observational settings. The results show that the learned graph is already substantially different from the reference graph in the anomaly-free settings, indicating that causal discovery on real-world data remains a challenging problem.
\begin{table}[b]
    \caption{Causal Chambers causal-discovery diagnostics. DAG Precision, Recall, F1 and SHD compare the learned graph to the reference graph. The clean-observational rows report a single clean-only discovery run, so no standard deviation is shown.}
\centering
\small
\begin{tabular}{llcccc}
\toprule
Strength & Setting & DAG Precision & DAG Recall & DAG F1 & SHD \\
\midrule
Mid & AE-filtered& 0.108 $\pm$ 0.017 & 0.137 $\pm$ 0.021 & 0.121 $\pm$ 0.018 & 58.350 $\pm$ 2.007 \\
Mid & Clean obs.& 0.143 & 0.179 & 0.159 & 56.0 \\
Strong & AE-filtered & 0.136 $\pm$ 0.043 & 0.188 $\pm$ 0.065 & 0.158 $\pm$ 0.051 & 59.450 $\pm$ 3.017 \\
Strong & Clean obs.&0.104 & 0.128 & 0.115 & 57.0 \\
\bottomrule
\end{tabular}

\label{tab:cc-causal-discovery-diagnostics}
\end{table}

\input{figs/additional_metrics_rca.tex}
\input{figs/additional_metric_rca_cc.tex}
For the main experiments, we provide Top-3 Recall, Top-5 Recall, Average Precision (Avg. Prec.), and AUC in Fig.~\ref{fig:additional_metrics_rca} and Fig.~\ref{fig:additional_metrics_rca_cc}. Both Avg. Prec. and AUC are computed individually for each anomalous sample and then averaged across all anomalous samples.

% \label{apx:metrics}

\subsubsection{Violated assumptions}
\label{apx:violations}
In this section, we report additional experiments that evaluate the performance of \ourmethod under violations of its assumptions, including confounding, controlled DAG misspecification, soft interventions, extreme outliers, heteroscedasticity and correlated root causes. 

\paragraph{Correlated root causes}
Instead of sampling mechanistic and measurement outliers independently, we introduce correlation by sampling pairs of variables for which we jointly inject mechanistic and measurement outliers. The results are shown in Fig.~\ref{fig:correlated_root_causes}. \ourmethodstar maintains strong performance in both localization and classification, indicating that violating the independence assumption with sparse but correlated perturbations does not significantly affect the performance.

\begin{figure}
     \begin{subfigure}[t]{0.45\linewidth}
        \centering
        \input{figs/top_k_recall_legend.tex}
    \end{subfigure}
    \hfill
       \begin{subfigure}[t]{0.45\linewidth}
        \centering
      \input{figs/classification_legend.tex}
    \end{subfigure}
  \begin{subfigure}[t]{0.45\textwidth}
    \centering
    \input{figs/rca_correlated.tex}
    \caption{Localization under correlated root causes (Top-$k$ recall, higher is better). We set $k$ to the true number of root causes in each sample.}
  \end{subfigure}
    \hfill
  \begin{subfigure}[t]{0.45\textwidth}
    \centering
       \input{figs/classification_correlated.tex}
    \caption{Classification under correlated root causes (accuracy, higher is better). } 
  \end{subfigure}%
  
  \caption{Performance of \ourmethod under correlated root causes.}
  \label{fig:correlated_root_causes}
\end{figure}

\paragraph{Dag misspecification}
We analyze the performance of \ourmethod and competitors under three types of DAG misspecification: (i)~edge additions, (ii)~edge removals, and (iii)~edge reversals. For edge additions, we randomly add edges between pairs of nodes that are not connected in the true DAG. For edge removals, we randomly remove existing edges from the true DAG. For edge reversals, we randomly select existing edges and reverse their direction. In each case, we introduce a fixed number of modifications to the DAG and evaluate the performance of \ourmethod and competitors on the modified graph. The results are shown in Fig.~\ref{fig:misspec_add}, Fig.~\ref{fig:misspec_remove} and Fig.~\ref{fig:misspec_reversal}. As expected, DAG misspecification leads to a decrease in performance for all methods that rely on a given DAG structure. In Tab.~\ref{tab:cc-graph-perturbation}, we report the top-$k$ recall of \ourmethod under different levels of DAG misspecification on the Causal Chambers dataset. The results show that \ourmethod is robust to moderate levels of DAG misspecification, with only a slight decrease in performance as the number of edge errors increases. However, as the number of edge errors becomes large, the performance of \ourmethod degrades more significantly, highlighting the importance of accurate DAG specification for optimal performance.

\begin{table}[t]
  \caption{Causal Chambers graph-perturbation sensitivity. The reference column reports top-$k$ recall. Add, delete, and reverse columns report paired changes in top-$k$ recall relative to the corresponding reference graph.}

\centering
\small
\begin{tabular}{llcccc}
\toprule
Strength & Edge errors & Reference & Add & Delete & Reverse \\
\midrule
Mid & 4 & 0.872 $\pm$ 0.027 & +0.003 $\pm$ 0.005 & -0.014 $\pm$ 0.015 & -0.033 $\pm$ 0.013 \\
Mid & 12 & 0.872 $\pm$ 0.027 & +0.006 $\pm$ 0.003 & -0.093 $\pm$ 0.049 & -0.102 $\pm$ 0.033 \\
Mid & 20 & 0.872 $\pm$ 0.027 & +0.006 $\pm$ 0.005 & -0.153 $\pm$ 0.052 & -0.139 $\pm$ 0.029 \\
Strong & 4 & 0.903 $\pm$ 0.039 & -0.002 $\pm$ 0.002 & -0.021 $\pm$ 0.022 & -0.028 $\pm$ 0.014 \\
Strong & 12 & 0.903 $\pm$ 0.039 & +0.016 $\pm$ 0.019 & -0.070 $\pm$ 0.030 & -0.067 $\pm$ 0.023 \\
Strong & 20 & 0.903 $\pm$ 0.039 & +0.018 $\pm$ 0.024 & -0.100 $\pm$ 0.018 & -0.108 $\pm$ 0.033 \\
\bottomrule
\end{tabular}
\label{tab:cc-graph-perturbation}
\end{table}

\begin{figure}
     \begin{subfigure}[t]{0.45\linewidth}
        \centering
        \input{figs/top_k_recall_legend.tex}
    \end{subfigure}
    \hfill
      \begin{subfigure}[t]{0.45\linewidth}
        \centering
      \input{figs/classification_full_legend.tex}
    \end{subfigure}
  \begin{subfigure}{0.45\textwidth}
    \centering
    \input{figs/rca_misspecification_add.tex}
    \caption{Localization under edge additions.}
  \end{subfigure}
  \hfill
  \begin{subfigure}{0.45\textwidth}
    \centering
       \input{figs/classification_misspecification_add.tex}
    \caption{Classification under edge additions.}
  \end{subfigure}%
  
  \caption{Performance of \ourmethod under edge additions.}
  \label{fig:misspec_add}
\end{figure}

\begin{figure}
     \begin{subfigure}[t]{0.45\linewidth}
        \centering
        \input{figs/top_k_recall_legend.tex}
    \end{subfigure}
    \hfill
       \begin{subfigure}[t]{0.45\linewidth}
        \centering
      \input{figs/classification_full_legend.tex}
    \end{subfigure}
  \begin{subfigure}{0.45\textwidth}
    \centering
    \input{figs/rca_misspecification_remove.tex}
    \caption{Localization under edge removals.}
  \end{subfigure}
  \hfill
  \begin{subfigure}{0.45\textwidth}
    \centering
       \input{figs/classification_misspecification_remove.tex}
    \caption{Classification under edge removals.}
  \end{subfigure}%
  
  \caption{Performance of \ourmethod under confounding and DAG misspecification.}
  \label{fig:misspec_remove}
\end{figure}

\begin{figure}
     \begin{subfigure}[t]{0.45\linewidth}
        \centering
        \input{figs/top_k_recall_legend.tex}
    \end{subfigure}
    \hfill
       \begin{subfigure}[t]{0.45\linewidth}
        \centering
      \input{figs/classification_full_legend.tex}
    \end{subfigure}
  \begin{subfigure}{0.45\textwidth}
    \centering
    \input{figs/rca_dag_misspecification.tex}
    \caption{Localization under edge reversals.}
  \end{subfigure}
  \hfill
  \begin{subfigure}{0.45\textwidth}
    \centering
       \input{figs/classification_misspecification.tex}
    \caption{Classification under edge reversals.}
  \end{subfigure}%
  
  \caption{Performance of \ourmethod under confounding and DAG misspecification.}
  \label{fig:misspec_reversal}
\end{figure}

\paragraph{Confounding}
We evaluate the performance of \ourmethod and its competitors under confounding under varying confounding levels at a shift strength of 3 standard deviations.  We add $p\%$ (of the default of $15$) new variables to the generating SCM, which are hidden from the algorithms. Methods that typically receive the ground truth DAG are now provided with the true subgraph among the observed variables. We restrict the evaluation of localization and classification to the observed variables, since root causes in unobserved variables can trivially neither be correctly localized nor classified. The results in Fig.~\ref{fig:confounding} show that performance decreases slightly for the variant that relies on causal discovery to learn the DAG, but the overall effect remains modest even when up to half of the variables in the generating SCM are hidden. 

\begin{figure}
     \begin{subfigure}[t]{0.45\linewidth}
        \centering
        \input{figs/top_k_recall_legend.tex}
    \end{subfigure}
    \hfill
       \begin{subfigure}[t]{0.45\linewidth}
        \centering
        \input{figs/classification_full_legend.tex}
    \end{subfigure}
  \begin{subfigure}{0.45\textwidth}
    \centering
    \input{figs/rca_confounding.tex}
    \caption{Localization under confounding}
  \end{subfigure}
  \hfill
  \begin{subfigure}{0.45\textwidth}
    \centering
       \input{figs/classification_confounding.tex}
    \caption{Classification under confounding}
  \end{subfigure}%
  
  \caption{Performance of \ourmethod under confounding.}
  \label{fig:confounding}
\end{figure}

\paragraph{Soft interventions}
We evaluate the performance of \ourmethod and its competitors under soft interventions. Instead of resampling from the shifted marginal distribution, we introduce additive perturbations to the original values. We vary the magnitude of the perturbations in standard deviations of the marginal, as we did for the hard interventions. Fig.~\ref{fig:soft_interventions} shows the results. \ourmethodstar and \ourmethod outperform the competitors in both localization and classification similarly to the hard intervention setting.
\begin{figure}
     \begin{subfigure}[t]{0.45\linewidth}
        \centering
        \input{figs/top_k_recall_legend.tex}
    \end{subfigure}
    \hfill
       \begin{subfigure}[t]{0.45\linewidth}
        \centering
        \input{figs/classification_legend.tex}
    \end{subfigure}
  \begin{subfigure}{0.45\textwidth}
    \centering
    \input{figs/rca_soft_top_k_recall.tex}
    \caption{Localization under soft interventions.}
  \end{subfigure}
  \hfill
  \begin{subfigure}{0.45\textwidth}
    \centering
       \input{figs/classification_soft.tex}
    \caption{Classification under soft interventions.}
  \end{subfigure}%
  
  \caption{Performance of \ourmethod under soft interventions.}
  \label{fig:soft_interventions}
\end{figure}

\paragraph{Endogenous intervention probabilities}
This experiment tests a violation of the assumption that intervention probabilities are independent of the pre-fault state. We first generate the clean latent sample \(X^0\) as in the default synthetic setting. Instead of selecting intervention targets independently of \(X^0\), we sample anomalous observations with probability proportional to \(\exp(\lambda(2\operatorname{rank}(X^0)-1))\), where \(\operatorname{rank}(X^0)\in[0,1]\) is the percentile rank of the pre-fault value. Thus, \(\lambda=0\) recovers state-independent anomaly occurrence, while larger \(\lambda\) makes anomalies increasingly concentrated in high-valued pre-fault states. The results show only a mild degradation as the dependence strength increases, suggesting that \ourmethodstar is robust to moderate endogeneity in the anomaly-generation process.
\begin{figure}
  \centering
  \input{figs/assumption_robustness_legend.tex}
  \vspace{-0.5em}
  \input{figs/assumption_robustness_endogeneity.tex}
  \caption{Performance of \ourmethodstar when intervention probabilities depend on the pre-fault state. We draw anomalies with sample weights proportional to \(\exp(\lambda(2\operatorname{rank}(X^0)-1))\), where \(\operatorname{rank}(X^0)\in[0,1]\) is the percentile rank of the counterfactual pre-fault state and \(\lambda\) controls the strength of the dependence. Changes are measured relative to the default synthetic setting, i.e., each score is the metric under the robustness setting minus the corresponding baseline score.}
  \label{fig:endogenous_intervention_probabilities}
\end{figure}

\paragraph{Measurement noise}
This experiment tests the effect of ordinary, non-anomalous measurement noise, which is not part of our formal model. After generating the latent causal variables and injected anomalies as in the default synthetic setting, we observe \(X_j=X_j^\ast+\eta_j\) instead of \(X_j=X_j^\ast\), with independent Gaussian noise \(\eta_j\) added to every measured variable. The noise strength is the ratio between the standard deviation of \(\eta_j\) and the standard deviation of the corresponding latent variable \(X_j^\ast\). Small amounts of additional measurement noise have little effect, while large noise levels increasingly obscure the causal signal needed for localization and type classification.
\begin{figure}
  \centering
  \input{figs/assumption_robustness_legend.tex}
  \vspace{-0.5em}
  \input{figs/assumption_robustness_measurement_noise.tex}
  \caption{Performance of \ourmethodstar under additional non-anomalous measurement noise. Noise strength is measured as the standard deviation of the added measurement noise divided by the standard deviation of the corresponding latent variable, so a value of \(1\) means equal standard deviations. Changes are measured relative to the default synthetic setting, i.e., each score is the metric under the robustness setting minus the corresponding baseline score.}
  \label{fig:measurement_noise}
\end{figure}

\paragraph{Extreme outliers}
We evaluate \ourmethod and its competitors under increasingly extreme outliers. 
This setting challenges an implicit assumption of \ourmethod : the learned 
conditional models and residual densities are most reliable near the support of 
the observed data, whereas very large shifts force the spline models to 
extrapolate far beyond the training range. As a simple remedy, we evaluate 
a clipped variant of \ourmethod that clamps each feature 
marginally to the interval given by the median $\pm\,6 \cdot 1.4826$ median 
absolute deviations, estimated from the observed contaminated data. 

Fig.~\ref{fig:extreme_outliers} shows that \ourmethodstar-\textsc{Clip} performs consistently 
best across all extreme outlier strengths. In contrast, the learned-DAG variant of
\ourmethod-\textsc{Clip} gradually degrades as the outlier strength increases and is eventually 
overtaken by \SMOOTHTRAVERSALSTAR. Most remaining methods degrade with increasing outlier strength, indicating that a simple robust marginal clipping 
step is sufficient to remove much of the extrapolation pathology. In contrast, for the classification task, we observe that for sufficiently extreme outliers \ourmethodstar no longer outperforms \MARGINALSTAR, indicating that for sufficiently strong outliers the simple marginal heuristic is no longer dominated by the more sophisticated conditional model.

\begin{figure}
     \begin{subfigure}[t]{0.45\linewidth}
        \centering
        \input{figs/top_k_recall_legend_clip.tex}
    \end{subfigure}
    \hfill
       \begin{subfigure}[t]{0.45\linewidth}
        \centering
        \input{figs/classification_legend_clip.tex}
    \end{subfigure}
  \begin{subfigure}{0.45\textwidth}
    \centering
    \input{figs/rca_extreme_outliers.tex}
    \caption{Localization under extreme outliers.}
  \end{subfigure}
  \hfill
  \begin{subfigure}{0.45\textwidth}
    \centering
       \input{figs/classification_extreme.tex}
    \caption{Classification under extreme outliers.}
  \end{subfigure}%
  
  \caption{Performance of \ourmethod under extreme outliers.}
  \label{fig:extreme_outliers}
\end{figure}
\paragraph{Heteroscedasticity}

We evaluate the performance of \ourmethod and its competitors under heteroscedastic noise.  We model heteroscedastic noise by letting the variance depend on a normalized signal profile $\phi_{j,i}$ derived from standardized deviations of $s_{j,i}= f_j\!\left(x_{\mathrm{Pa}(j),i}\right)$. The noise scale is $\sigma_{j,i} = \sigma_0[(1-h) + h\,\phi_{j,i}]$, where $h\in[0,1]$ interpolates between homoscedastic ($h=0$) and fully heteroscedastic ($h=1$) noise. We set $\sigma_0=0.1$ as for the other experiments and $\rho=0.25$.
We keep the mean shift at 3 standard deviations and vary the heteroscedasticity strength parameter $h$ (0.0, 0.25, 0.5, 0.75, 1.0). The results in Fig.~\ref{fig:heteroscedastic} show that CALI remains robust under heteroscedastic noise, and the causal discovery variant even benefits from the increased variance. 

\begin{figure}
     \begin{subfigure}[t]{0.45\linewidth}
        \centering
        \input{figs/top_k_recall_legend.tex}
    \end{subfigure}
    \hfill
       \begin{subfigure}[t]{0.45\linewidth}
        \centering
        \input{figs/classification_legend.tex}
    \end{subfigure}
  \begin{subfigure}{0.45\textwidth}
    \centering
    \input{figs/rca_hetero.tex}
    \caption{Localization under heteroscedastic noise.}
  \end{subfigure}
  \hfill
  \begin{subfigure}{0.45\textwidth}
    \centering
       \input{figs/classification_hetero.tex}
    \caption{Classification under heteroscedastic noise.}
  \end{subfigure}%
  
  \caption{Performance of \ourmethod under heteroscedastic noise.}
  \label{fig:heteroscedastic}
\end{figure}

\subsubsection{Runtime}
\label{apx:runtime}
In Fig.~\ref{fig:runtime}, we report the runtime of \ourmethod and the competing methods. 
We evaluate scaling in two regimes: first fixing the number of features to \(d=15\) and varying the sample size, and then fixing \(N=2000\) and varying the number of features. 
\ourmethod is slower than the baselines because it evaluates feature-level intervention candidates for all samples, including measurement candidates whose likelihoods require additional Markov-blanket evaluations. 
For the learned-graph variant, the reported runtime also includes causal discovery.

When runtime is a concern, a simple practical acceleration is to restrict assignment updates to samples that are most unlikely under the fitted clean density. 
In our current implementation, \ourmethod considers candidate root causes for all samples whose updates improve the global objective. 
One could instead score only the top-\(k\) samples with largest negative log-likelihood under the clean model, or use this set as an initial active set, thereby reducing the dominant candidate-evaluation cost while preserving focus on the most anomalous observations.

\begin{figure}
  
  \begin{subfigure}[t]{0.33\textwidth}
      \centering
  \input{figs/sample_runtime.tex}
  
  \caption{Sample size}
  \end{subfigure}%
  \begin{subfigure}[t]{0.33\textwidth}
    \centering
    \input{figs/feature_runtime.tex}
    \caption{Feature size}
  \end{subfigure}%
  \begin{subfigure}[t]{0.33\textwidth}
    \centering
       \input{figs/runtime_legend.tex}
    \end{subfigure}
    \caption{Runtimes for increasing sample and feature size.}
    \label{fig:runtime}
\end{figure}

\subsubsection{Anomaly detection}
We compare against \AUTOENCODER~\citep{sakurada:2014:aead}, \DEEPSVDD~\citep{ruff:2018:deepsvdd}, \IFOREST~\citep{liu:2008:isolation}, \COPOD~\citep{li:2020:copod}, \LOF~\citep{breunig:2000:lof}, and \NORMALFLOW~\citep{jimenez:2015:normalizingflows}. 

We report AUC and Average Precision (AP) for synthetic data and the Causal Chamber dataset (Fig.~\ref{fig:synthetic_sample_full}). Despite primarily being designed for characterization, \ourmethod also performs great at detection.
  
\input{figs/sample_ad_full}

\subsubsection{Additional RCA experiments}
\label{apx:additional_rca}
In this section, we provide various additional RCA experiments to further analyze the performance of \ourmethod and its competitors.
\paragraph{Varying parameters in synthetic data generation.}
We repeat the synthetic experiment but vary either sample size, feature size, edge probability, or the expected number of root causes per expected sample. We use the same default settings as previously ($N=2000,D=15,p_e=0.3$) and vary one of the factors. To analyze how the methods work under multiple root causes, we do one additional experiment we no longer use the independent Bernoulli draws, but instead randomly place $\mathit{rc}_\mathit{count}$ root causes in 10\% of samples, deciding with a coin flip whether we create a measurement or mechanistic outlier.  We present the results in Fig.~\ref{fig:vary}.   

\paragraph{Individual outlier types}
We do an experiment under the default conditions, but only add one outlier type to evaluate how well the methods work for either type. We present the results in Fig.\ref{fig:only_mech_meas}.
While many methods work well when the data contains only measurement outliers, \ourmethodstar, \SCM and \ourmethod perform best when it contains only mechanistic outliers.

\begin{figure}
  \begin{subfigure}[t]{\textwidth}
    \centering
        \input{figs/varying_legend.tex}
    \end{subfigure}
  \begin{subfigure}[t]{0.25\textwidth}
    \centering
    \input{figs/sample_top_k_recall.tex}
  \end{subfigure}%
  \begin{subfigure}[t]{0.25\textwidth}
    \centering
    \input{figs/feature_top_k_recall.tex}
  \end{subfigure}%
  \begin{subfigure}[t]{0.25\textwidth}
    \centering
    \input{figs/edge_prob_top_k_recall.tex}
  \end{subfigure}%
  \begin{subfigure}[t]{0.25\textwidth}
    \centering
    \input{figs/rc_count_top_k_recall.tex}
  \end{subfigure}%

  \caption{Root cause localization under varying parameters of the synthetic data generation.}
  \label{fig:vary}
\end{figure}%

\paragraph{Realistic RCA evaluation.}
To assess how many root causes can be realistically identified in practice, we conduct an additional experiment where we only evaluate RCA methods on samples that are actually detected as anomalous by a sample-level detector. This evaluation captures the practically relevant question of how many true root causes can be recovered when both anomaly detection and localization are imperfect, rather than assuming oracle knowledge of anomalous samples
We evaluate all competing RCA methods in combination with \AUTOENCODER, which is the best-performing sample-level detector on the synthetic data, in order to obtain fixed binary predictions. Specifically, we first use the true sample contamination ratio to determine the corresponding score quantile of \AUTOENCODER, yielding fixed anomaly predictions at the sample level. Next, for the samples predicted as anomalous, we apply the same procedure at the feature level using the RCA scores: we compute the ratio of root causes among truly anomalous samples and use it to select the corresponding quantile of the RCA scores. This process results in a binary prediction matrix indicating predicted root causes for each detected sample.

For \ourmethod, this post-processing step is unnecessary, as the method directly outputs feature assignments through the learned model. We report Precision, Recall, and F1 score, with results shown in Fig.~\ref{fig:realistic}. On synthetic data, \ourmethodstar and \ourmethod outperform all competing approaches by a large margin, achieving F1 scores of up to 0.8 for shifts of at least four standard deviations. This is particularly encouraging because, unlike competing methods, \ourmethod does not rely on thresholding based on the true contamination ratio, which is typically unknown in practice. On the Causal Chamber dataset, \ourmethodstar and \ourmethod likewise attain the best overall performance.

\begin{figure}
  
\begin{subfigure}[t]{0.33\textwidth}
    \centering
    \input{figs/rca_F1.tex}

  \end{subfigure}%
  \begin{subfigure}[t]{0.33\textwidth}
    \centering
    \input{figs/rca_cc_F1.tex}

  \end{subfigure}%
  \begin{subfigure}[t]{0.33\textwidth}
    \centering
        \input{figs/realistic_legend.tex}
    \end{subfigure}
    \caption{F1 scores of fixed assignments ($\uparrow$ better). Left: Synthetic data, Right: Causal Chamber dataset.}
    \label{fig:realistic}
\end{figure}
\paragraph{RCA for individual outlier types.}

Here we generate data with only one outlier type at a time to evaluate how well the methods can localize either mechanistic or measurement outliers. We use the same default settings as previously ($N=2000,D=15,p_e=0.3$), and either only add measurement or mechanistic outliers. The results are presented in Fig.~\ref{fig:only_mech_meas}.

\begin{figure}[t]
  \begin{subfigure}[t]{0.33\textwidth}
    \centering
    \input{figs/meas_top_k_recall.tex}
    \caption{Measurement outliers}
  \end{subfigure}%
  \begin{subfigure}[t]{0.33\textwidth}
    \centering
    \input{figs/mech_top_k_recall.tex}
    \caption{Mechanistic outliers}

  \end{subfigure}%
  \begin{subfigure}[t]{0.33\textwidth}
    \centering
        \input{figs/individual_legend.tex}
    \end{subfigure}
    \caption{Root cause localization for the individual outlier types.}
    \label{fig:only_mech_meas}
\end{figure}

\subsubsection{Sachs dataset}

The full results of the Sachs experiment are reported in Table~\ref{tab:sachs}. Clearly, the dataset is quite challenging, as all methods achieve much lower RCA performance compared to the synthetic data or the Causal Chamber dataset. Somewhat surprisingly, \ourmethod outperforms \ourmethodstar on most root-cause localization metrics, while \MARGINAL and \ourmethod achieve the best overall classification performance. Notably, for the root causes that can be detected, the classification accuracy of 84\% is even better than for the synthetic data. We attribute the fact that \ourmethod outperforms \ourmethodstar to the factor that the data may not be entirely faithful to the assumed ground-truth DAG. Since it is known that the system contains feedback loops, the agreed-upon DAG is only an approximation. The fact that \MARGINALSTAR performs significantly worse than \MARGINAL further supports this hypothesis.

\begin{table}
  \caption{Root cause localization and classification on the Sachs dataset. The highest performance is bolded.}
  \resizebox{\textwidth}{!}{%
    \input{tables/sachs_table.tex}
  }

  \label{tab:sachs}
\end{table}

\subsubsection{Classification experiments}
In this section we provide additional experiments related to the classification of outlier types.
\paragraph{\ourmethod and \MARGINAL variants}
In Fig.~\ref{fig:classification_variants}, we compare the classification performance of \ourmethod and \MARGINAL with a learned causal DAG. The experimental conditions are identical to the main paper. Surprisingly, the performance on the Sachs dataset is much better than on Causal Chamber despite weaker interventions, possibly because the learned DAG is more accurate. Overall, \ourmethod outperforms \MARGINAL on synthetic data, whereas no method is consistently better on the real data. Since this is very different from the results with the ground-truth DAG, we attribute this to the learned DAG restricting classifications to trivial cases, which \MARGINAL can handle well.

\begin{figure}
  \begin{subfigure}[t]{0.5\textwidth}%
        \centering
    \input{figs/classification_synth_learned.tex}
    \caption{Synthetic Data}
  \end{subfigure}%
  \begin{subfigure}[t]{0.5\textwidth}%
        \centering
      \input{figs/classification_cc_learned.tex}
      \caption{Real data} 
  \end{subfigure}%
  \caption{Classification performance of \ourmethod and \MARGINAL with learned causal DAG.}
  \label{fig:classification_variants}
\end{figure}
\paragraph{Sample classification}
\label{apx:classification}

Instead of inspecting classification performance at the level of detected root causes, as done in the main paper, we also consider classification at the sample level. In our causal setting, each sample may contain measurement failure(s), mechanistic failure(s), or both. However, the presence of unidentifiable sink nodes complicates a direct sample-level categorization, as failures at such nodes cannot be causally distinguished without additional assumptions. Representing all possible combinations would therefore require introducing several additional and potentially unintuitive classes.

\cite{larosa:2022:separating} propose a classification method for a related but simpler problem. Their approach distinguishes between \emph{sensor anomalies}, defined as failures affecting exactly one observed feature, and \emph{process anomalies}, defined as failures affecting multiple features, in addition to the \emph{normal} class. Under this definition, mechanistic failures at sink nodes are interpreted as sensor anomalies, and multiple independent mechanistic failures are classified as process anomalies. Consequently, their method does not aim to identify causal failure mechanisms in the data-generating process, but rather distinguishes whether none, exactly one, or more than one marginal distribution of the observed variables has changed.

Nevertheless, we compare \ourmethod{} with the approach of LaRosa et al.\ under their classification scheme to assess whether incorporating causal structure also improves performance on this simpler sample-level task. Sample-level predictions for \ourmethod{} are obtained by aggregating the inferred root-cause assignments. Specifically, whenever \ourmethod{} assigns a mechanistic root cause to a node with successors, or assigns multiple root causes, we classify the sample as a \emph{process anomaly}. When exactly one root cause is assigned and it corresponds either to a measurement failure or to a failure at a sink node, we classify the sample as a \emph{sensor anomaly}. Samples without any assigned root cause are classified as \emph{normal}.

The method of LaRosa et al.\ is originally designed for time-series data and is therefore adapted to the tabular setting. First, we separate the data into clean and anomalous samples using an autoencoder. Next, we generate 100 random groups of features of size three. These groups serve as inputs to gradient boosting prediction models, one for each feature, resulting in a total of $100 \times D$ predictors (1{,}500 predictors for $D = 15$). Classification is performed by thresholding the residuals of these predictors into normal and abnormal; we evaluate multiple significance levels $\alpha \in \{0.001, 0.01, 0.05, 0.1\}$.

Based on these binary residual predictions, LaRosa et al.\ define the following classification rule: if no residual is classified as abnormal, the sample is labeled \emph{normal}. If there exists a feature $i$ for which all corresponding groups predict abnormal residuals, and no other feature $j$ (predicted from any group not containing $i$) is classified as abnormal, the sample is labeled \emph{sensor anomaly}. In all remaining cases, the sample is classified as \emph{process anomaly}.

We present the results in Figure~\ref{fig:larosa}.

\begin{figure}
  \begin{subfigure}[t]{0.33\textwidth}%
        \centering
    \input{figs/sample_class_accuracy.tex}
    \caption{Accuracy ($\uparrow$ better).}
  \end{subfigure}%
  \begin{subfigure}[t]{0.33\textwidth}%
        \centering
      \input{figs/sample_class_macro_f1.tex}
      \caption{Macro F1 score ($\uparrow$ better).}

  \end{subfigure}%
  \begin{subfigure}[t]{0.33\textwidth}%
        \centering
      \input{figs/sample_class_clean_f1.tex}
    \caption{\emph{Normal} F1 score ($\uparrow$ better).}
  \end{subfigure}

  \begin{subfigure}[t]{0.33\textwidth}%
    \centering
      \input{figs/sample_class_sensor_f1.tex}
    \caption{\emph{Sensor anomaly} F1 score ($\uparrow$ better).}

  \end{subfigure}%
  \begin{subfigure}[t]{0.33\textwidth}%
    \centering
      \input{figs/sample_class_process_f1.tex}
    \caption{\emph{Process anomaly} F1 score ($\uparrow$ better).}

  \end{subfigure}%
    \begin{subfigure}[t]{0.33\textwidth}%
      \centering
      \input{figs/sample_classification_legend.tex}
  \end{subfigure}
  \caption{Classification into the classes \emph{Normal}, \emph{Sensor Anomaly}, \emph{Process Anomaly}. Accuracy, macro F1, and per-class F1 scores are reported due to class imbalance.}
  \label{fig:larosa}
\end{figure}

\paragraph{Influence of degree}
To analyze classification behavior under different structural conditions, we begin with a three-variable chain
$X_1 \to X_2 \to X_3$ and introduce both a measurement and a mechanistic outlier at $X_2$. We then progressively add
parent and child nodes to $X_2$ and evaluate how classification accuracy changes. For each graph structure, we
generate random nonlinear functions and noise following the setup of the previous synthetic experiments.
Classification accuracy is averaged over 50 runs, counting only those runs in which both root causes are correctly
identified. Fig.~\ref{fig:class_degree} summarizes the results.

Surprisingly, \MARGINALSTAR outperforms \ourmethodstar when $X_2$ has two parents and two children; however, its
performance degrades once an additional parent and child are added. This is probably because the optimal significance threshold to determine marginal outlierness depends on the degree. This suggests that if the outliers are strong enough that propagation can easily be detected, we need the right threshold, which is impossible to choose in the unsupervised setting. In contrast, the performance of \ourmethodstar
improves consistently as the degree of the outlier node increases, without requiring the selection of a
significance threshold.

\begin{figure}

  \begin{subfigure}[t]{0.5\textwidth}
     \centering
  \input{figs/classification_toy.tex}
  \end{subfigure}

  \caption{Influence of node degree on classification accuracy.}
  \label{fig:class_degree}
\end{figure}

\paragraph{Confidence}
Last but not least, we look at the confidence scores produced by \ourmethod. Specifically, we evaluate how well the confidence scores rank correct classifications, by evaluating the AUC and calculating the Spearman correlation. We present the results in Fig.~\ref{fig:confidence}.

\begin{figure}
  \begin{subfigure}[t]{0.25\textwidth}%
    \input{figs/class_auc.tex}
    \caption{AUC ($\uparrow$ better).}
  \end{subfigure}
  \begin{subfigure}[t]{0.25\textwidth}%
      \input{figs/class_spear_rho.tex}
      \caption{Spearman's $\rho$ ($\uparrow$ better)}

  \end{subfigure}
  \begin{subfigure}[t]{0.25\textwidth}%
      \input{figs/class_spear_p.tex}
    \caption{Significance of Spearman correlation ($\downarrow$ better).}

    \end{subfigure}%
    \begin{subfigure}[t]{0.25\textwidth}%
      \centering
      \input{confidence_legend.tex}
    \end{subfigure}
  \caption{We evaluate the confidence scores of \ourmethod{} by evaluating the ranking of true positive predictions, in terms of AUC and Spearman correlation. Overall the confidence scores produce a clear but imperfect signal, which is likely due to their heuristic nature. The confidence scores are weakly positively correlated with true positive predictions, and clearly significant at $\alpha=0.05$.}
  \label{fig:confidence}
\end{figure}

\subsubsection{Causal Discovery}
\label{apx:causal_discovery}
In this experiment, we evaluate the performance of \ourmethod on DAGs learned using different causal discovery algorithms, namely \TOPIC \citep{xu:2025:topic}, \CAM \citep{buhlmann:2014:cam}, \SCORE \citep{rolland:2022:score}, \NOGAM \citep{montagna:2023:causal}, and \NOTEARSMLP \citep{zheng:2020:learning}. We run the algorithms under three different settings on our synthetic data. First, they have access to the clean data. Second, we apply filtering to the corrupted data using either an autoencoder or robust z-scores based on the median and MAD. In both cases, we remove values outside the 0.95 quantile and then learn the DAG from the filtered data. Overall the \AUTOENCODER prefiltering performs better than z-scores, and \NOGAM achieves the best downstream performance. We report all methods with \AUTOENCODER prefiltering, and the different prefiltering strategies with \NOGAM in Fig.~\ref{fig:causal_discovery}.

\begin{figure}
  \begin{subfigure}[t]{\textwidth}
    \centering
        \input{figs/causal_discovery_legend.tex}
    \end{subfigure}
  \begin{subfigure}[t]{0.25\textwidth}
  \input{figs/causal_cd_top_k_recall.tex}
  \end{subfigure}%
   \begin{subfigure}[t]{0.25\textwidth}
  \input{figs/causal_cd_ap.tex}
  \end{subfigure}%
  \begin{subfigure}[t]{0.25\textwidth}
  \input{figs/causal_cd_accuracy.tex}
  \end{subfigure}%
   \begin{subfigure}[t]{0.25\textwidth}
  \input{figs/causal_cd_shd.tex}
  \end{subfigure}
  \begin{subfigure}[t]{\textwidth}
    \centering
        \begin{tikzpicture}[baseline]

            \begin{axis}[
                hide axis,
                width=\linewidth,height=4cm,
                tick label style={font=\small},
                legend style={ at={(0.3,0.0)}, anchor=center, font=\small},
                legend cell align=left,legend columns=5,legend image post style={xscale=0.5} ]
            \addplot[draw=none, forget plot] coordinates {(0) (0,0)};
            \addlegendimage{line width = 2pt, color=NOGAMColor!70!black}
            \addlegendentry{\NOGAM + \AUTOENCODER}
             \addlegendimage{line width = 2pt, color=TOPICColor!70!black}
            \addlegendentry{\NOGAM + Oracle}
             \addlegendimage{line width = 2pt, color=NOTEARSColor!70!black}
            \addlegendentry{\NOGAM + z scores}        
            \end{axis}
        \end{tikzpicture}
    \end{subfigure}
  \begin{subfigure}[t]{0.25\textwidth}
  \input{figs/causal_cd_top_k_recall_nogam.tex}
  \end{subfigure}%
   \begin{subfigure}[t]{0.25\textwidth}
  \input{figs/causal_cd_ap_nogam.tex}
  \end{subfigure}%
  \begin{subfigure}[t]{0.25\textwidth}
  \input{figs/causal_cd_accuracy_nogam.tex}
  \end{subfigure}%
   \begin{subfigure}[t]{0.25\textwidth}
  \input{figs/causal_cd_shd_nogam.tex}
  \end{subfigure}

  \caption{Performance of \ourmethod when combined with different causal discovery algorithms. Top-$k$ Recall, and Avg. Prec. assess the downstream root cause localization, Accuracy assesses outlier type classification and Structural Hamming Distance (SHD) similarity to the true DAG. Top: Different algorithms with \AUTOENCODER prefiltering. Bottom: \NOGAM with different prefiltering strategies.}
  \label{fig:causal_discovery}
\end{figure}

\subsubsection{Non-causal RCA}
\label{apx:non_causal_rca}

For non-causal RCA methods without a notion of causal structure, we evaluate \AESHAP, \AEREC, \MOE \citep{angiulli:2024:m2oe}, and \DICE \citep{mothilal:2020:dice}. We present the results in Fig.~\ref{fig:rca_non_causal}. \AESHAP performs best with \MOE coming in second. \DICE clearly shows the worst performance. We want to emphasize that none of these methods are designed for the task that we evaluate them on, and we include them here for completeness. Our main experiments indicate that \AESHAP still performs surprisingly well, even without access to the causal structure or causal guarantees.

\begin{figure}
    
    \begin{subfigure}[t]{\textwidth}
        \centering
        \begin{tikzpicture}[baseline]

            \begin{axis}[
                hide axis,
            width=\linewidth,height=4cm,xtick pos=left,ymax=1.05, ytick pos=left,tick align=outside,minor x tick num=0,minor y tick num=0,
            , width=\textwidth,ymin=0,
            axis x line=bottom,
                axis y line=left,   
                xtick={3,4,5},
                xmax = 5.2,
                xmin = 2.8,
                ymin=0.55,
                ylabel style={font=\small}, 
                ylabel= {Top-$k$ Recall},
                xlabel style={font=\small},
                xlabel={Outlier strength},
            yticklabel style={/pgf/number format/fixed},legend style={
                draw=black,      
                fill=white,        
                at={(0.98,0.04)},  
                anchor=south east,
                font=\small,
                cells={align=center}
            }, ymajorgrids=true,
                grid style={dotted,gray},
                tick label style={font=\small},
            legend style={ at={(0.0,0.0)}, anchor=center, font=\small},
            legend cell align=left,legend columns=2,legend image post style={xscale=0.25} ]
            \addplot[draw=none, forget plot] coordinates {(0) (0,0)};
  
            \addlegendimage{line width=2pt, color=AEColor!70!black}
            \addlegendentry{\AUTOENCODER + \SHAP}
            \addlegendimage{line width=2pt, color=STColor!70!black}
            \addlegendentry{\AUTOENCODER + Rec.}
              \addlegendimage{line width=2pt, color=MOEColor!70!black}
            \addlegendentry{\MOE}
              \addlegendimage{line width=2pt, color=DICEColor!70!black}
            \addlegendentry{\DICE}
          
            \end{axis}
        \end{tikzpicture}
    \end{subfigure}

      \begin{subfigure}[t]{0.5\textwidth}

        \input{figs/top_k_recall_non_causal}

    \end{subfigure}%
      \begin{subfigure}[t]{0.5\textwidth}

        \input{figs/rca_non_causal_zoom}

    \end{subfigure}%

    \caption{Root cause localization ($\uparrow$ better) on synthetic data.}
    \label{fig:rca_non_causal}
\end{figure}%

\subsection{Broader Impact}

Distinguishing measurement errors from genuine mechanism shifts can have positive impact in domains where anomaly responses differ substantially, such as scientific measurement, engineering systems, and operations monitoring. Correctly identifying measurement anomalies may support data correction or sensor maintenance, while mechanistic anomalies may indicate genuine changes in the underlying system.

At the same time, incorrect classifications could lead to inappropriate interventions, for example correcting away a genuine process change or escalating a mere measurement error. This risk is particularly important in high-stakes applications. The method should therefore be used as decision support rather than as an automated decision-making system, and its conclusions should be interpreted in light of the assumed or estimated causal graph and the modeling assumptions discussed in the paper.

%% file: tables/hyperpar_detection.tex
\begin{tabular}{lcccc}
\toprule
Method & Batch size ($\mathit{bs}$) & Epochs ($e$) & Estimators ($\mathit{est}$) & Neighbors ($\mathit{neigh}$) \\
\midrule
\AUTOENCODER & 32 & 10  & –   & – \\
\NORMALFLOW  & 32 & 10  & –   & – \\
\DEEPSVDD    & 128 & 100 & –   & – \\
\IFOREST     & –   & –   & 300 & – \\
\LOF         & –   & –   & –   & 10 \\
\bottomrule
\end{tabular}

%% file: tables/hyperpar_rca.tex
\begin{tabular}{lcccc}
\toprule
Method & Batch size ($\mathit{bs}$) & Epochs ($e$) & KNN neighbors ($K$) & Detector \\
\midrule
\MOE      & 32  & 500 & 100 &-\\
\SCOREORDER & - & - & -&\textsc{MedianCDFQuantile}\\
\SMOOTHTRAVERSAL & - & - & -&\textsc{MedianCDFQuantile}\\
\bottomrule
\end{tabular}

%% file: tables/sachs_table.tex
\begin{tabular}{l|lllllllll}
\toprule
\diagbox{Metric}{Method} & \AEREC & \AESHAP & \MARGINALSTAR & \MARGINAL & \ourmethod & \ourmethodstar & \SCM &\SCOREORDER & \SMOOTHTRAVERSAL \\

\midrule
Top-$3$ Recall & 0.412 $\pm$ 0.050 & 0.435 $\pm$ 0.041 &  &  & \textbf{0.438 $\pm$ 0.057} & 0.402 $\pm$ 0.042 & 0.412 $\pm$ 0.049 & 0.414 $\pm$ 0.043 & 0.347 $\pm$ 0.036 \\
Top-$5$ Recall & 0.548 $\pm$ 0.052 & 0.573 $\pm$ 0.039 &  &  & \textbf{0.620 $\pm$ 0.075} & 0.587 $\pm$ 0.068 & 0.544 $\pm$ 0.054 & 0.550 $\pm$ 0.041 & 0.564 $\pm$ 0.048 \\
Top-$k$ Recall & 0.269 $\pm$ 0.026 & \textbf{0.279 $\pm$ 0.046} &  &  & 0.236 $\pm$ 0.036 & 0.192 $\pm$ 0.036 & 0.236 $\pm$ 0.036 & 0.240 $\pm$ 0.032 & 0.187 $\pm$ 0.032 \\
Avg. Prec. & 0.416 $\pm$ 0.027 & \textbf{0.430 $\pm$ 0.034} &  &  & 0.410 $\pm$ 0.034 & 0.375 $\pm$ 0.032 & 0.399 $\pm$ 0.030 & 0.376 $\pm$ 0.027 & 0.337 $\pm$ 0.024 \\
AUC & 0.586 $\pm$ 0.039 & 0.606 $\pm$ 0.032 &  &  & \textbf{0.623 $\pm$ 0.042} & 0.604 $\pm$ 0.043 & 0.591 $\pm$ 0.028 & 0.584 $\pm$ 0.033 & 0.558 $\pm$ 0.024 \\
Accuracy &  &  & 0.663 $\pm$ 0.077 & 0.847 $\pm$ 0.075 & \textbf{0.898 $\pm$ 0.097} & 0.860 $\pm$ 0.090 &  &  &  \\
\bottomrule
\end{tabular}

%% file: checklist.tex
\section*{NeurIPS Paper Checklist}

\begin{enumerate}

\item {\bf Claims}
    \item[] Question: Do the main claims made in the abstract and introduction accurately reflect the paper's contributions and scope?
    \item[] Answer: \answerYes{} % Replace by \answerYes{}, \answerNo{}, or \answerNA{}.
    \item[] Justification:  The abstract and introduction accurately summarize the paper's scope: they claim a causal formalization of measurement and mechanistic anomalies, an identifiability characterization, and an efficient method for localization and type classification. These claims are supported by the theoretical analysis, experiments on synthetic and real datasets, learned-DAG evaluations, and the case study. The assumptions are also reflected in the paper, including the need for a given or estimated causal DAG and the unsupervised in-sample setting.

\item {\bf Limitations}
    \item[] Question: Does the paper discuss the limitations of the work performed by the authors?
    \item[] Answer: \answerYes{} % Replace by \answerYes{}, \answerNo{}, or \answerNA{}.
    \item[] Justification: Sec.~\ref{sec:conclusion} discusses indistinguishability of anomaly types at sink nodes and classification uncertainty when assignment-conditional distributions overlap, the assumptions of exogenous hard interventions and pure-leaf measurements without ordinary measurement error, dependence on a given or estimated DAG and model specification, the hard-assignment approximation that precludes exact posteriors, and the restriction to tabular data. Appx.~\ref{apx:violations} evaluates robustness to assumption violations, including endogenous intervention probabilities and ordinary measurement noise.

\item {\bf Theory assumptions and proofs}
    \item[] Question: For each theoretical result, does the paper provide the full set of assumptions and a complete (and correct) proof?
    \item[] Answer: \answerYes{} % Replace by \answerYes{}, \answerNo{}, or \answerNA{}.
 \item[] Justification: The paper states the assumptions for its theoretical results in the corresponding theorem statements, including the fixed DAG/model setting, the reduced sink-node parameterization, mixture-weight identifiability, open-parameter assumptions, and rank conditions. Full proofs are provided in the appendix for the identifiability result and the generic identifiability results for linear Gaussian and fixed-degree polynomial Gaussian SCMs. The main text summarizes the consequences of these results and refers readers to the appendix for the formal statements and proofs.

    \item {\bf Experimental result reproducibility}
    \item[] Question: Does the paper fully disclose all the information needed to reproduce the main experimental results of the paper to the extent that it affects the main claims and/or conclusions of the paper (regardless of whether the code and data are provided or not)?
   
    \item[] Answer: \answerYes{}
    \item[] Justification: The paper provides the information needed to reproduce the main experimental results. It describes the synthetic SCM generation, outlier injection procedure, real datasets, evaluation metrics, graph assumptions, learned-DAG setting, and the baseline comparison protocol in the main text. Additional details on data generation and processing, baseline implementations and hyperparameter selection, the CALI algorithm, complexity, and robustness experiments are provided in the appendix. The datasets used are either synthetic, or publicly available.   

\item {\bf Open access to data and code}
    \item[] Question: Does the paper provide open access to the data and code, with sufficient instructions to faithfully reproduce the main experimental results, as described in supplemental material?
    \item[] Answer: \answerYes{} % Replace by \answerYes{}, \answerNo{}, or \answerNA{}.
\item[] Justification: We provide the source code necessary to reproduce all experiments at \url{https://github.com/hsuhr99/measurement_mechanistic_rca}, including a dockerized environment and a README file with instructions for running the experiments. The repository also includes scripts for generating synthetic data and loading and preprocessing the public real-world datasets used in the paper.

\item {\bf Experimental setting/details}
    \item[] Question: Does the paper specify all the training and test details (e.g., data splits, hyperparameters, how they were chosen, type of optimizer) necessary to understand the results?
    \item[] Answer: \answerYes{} % Replace by \answerYes{}, \answerNo{}, or \answerNA{}.
\item[] Justification: The paper specifies the main experimental settings needed to interpret the results, including the synthetic data-generating process, outlier injection procedure, real datasets, evaluation metrics, graph assumptions, learned-DAG setting, and comparison protocol. Additional details on preprocessing, baseline implementations, hyperparameter selection, algorithmic choices, and robustness experiments are provided in the appendix and supplementary code.

\item {\bf Experiment statistical significance}
    \item[] Question: Does the paper report error bars suitably and correctly defined or other appropriate information about the statistical significance of the experiments?
    \item[] Answer: \answerYes{} % Replace by \answerYes{}, \answerNo{}, or \answerNA{}.
    \item[] Justification:The main quantitative plots report shaded regions showing mean \(\pm 1.96\) standard errors across 20 independent runs.

\item {\bf Experiments compute resources}
    \item[] Question: For each experiment, does the paper provide sufficient information on the computer resources (type of compute workers, memory, time of execution) needed to reproduce the experiments?
    \item[] Answer: \answerYes{} % Replace by \answerYes{}, \answerNo{}, or \answerNA{}.
    \item[] Justification: We mention in the appendix, that all experiments were run on a CPU cluster. Each individual run used a single CPU core, and no GPU acceleration was required. Peak memory usage was below 8 GB per run. The reported experiments required approximately 200 CPU-hours in total, dominated by repeated synthetic and robustness runs. Representative individual runtimes for the main methods are reported in Appx.~\ref{apx:runtime}.

    
\item {\bf Code of ethics}
    \item[] Question: Does the research conducted in the paper conform, in every respect, with the NeurIPS Code of Ethics \url{https://neurips.cc/public/EthicsGuidelines}?
    \item[] Answer: \answerYes{} % Replace by \answerYes{}, \answerNo{}, or \answerNA{}.
\item[] Justification: We have reviewed the NeurIPS Code of Ethics and believe the research conforms to it. The paper uses synthetic data and publicly available benchmark datasets and does not involve human-subject experiments or sensitive personal data collection. Potential limitations and appropriate scope of the method are discussed in the paper.    

\item {\bf Broader impacts}
    \item[] Question: Does the paper discuss both potential positive societal impacts and negative societal impacts of the work performed?
   \item[] Answer: \answerYes{}
\item[] Justification: The appendix includes a broader-impact discussion. It describes potential positive impacts from better distinguishing measurement errors from genuine mechanism shifts, and potential negative impacts from incorrect classifications leading to inappropriate interventions, especially in high-stakes domains. It also emphasizes that the method should be used as decision support and interpreted in light of the causal and modeling assumptions.
    
\item {\bf Safeguards}
    \item[] Question: Does the paper describe safeguards that have been put in place for responsible release of data or models that have a high risk for misuse (e.g., pre-trained language models, image generators, or scraped datasets)?
    \item[] Answer: \answerNA{} % Replace by \answerYes{}, \answerNo{}, or \answerNA{}.
\item[] Justification: The paper does not release a high-risk model, scraped dataset, or dual-use generative system requiring safeguards. The released code implements a methodological framework for anomaly characterization, and the experiments use synthetic data and publicly available benchmark datasets.   

\item {\bf Licenses for existing assets}
    \item[] Question: Are the creators or original owners of assets (e.g., code, data, models), used in the paper, properly credited and are the license and terms of use explicitly mentioned and properly respected?
\item[] Answer: \answerYes{}
\item[] Justification: We cite the original sources for all external datasets, benchmarks, and methods used in the experiments. Dataset URLs, access information, and available license or terms-of-use information are provided in the paper or supplementary material. The released code environment lists package versions for software dependencies, and we use these assets according to their stated terms of use.

\item {\bf New assets}
    \item[] Question: Are new assets introduced in the paper well documented and is the documentation provided alongside the assets?
   \item[] Answer: \answerYes{}
\item[] Justification: We provide the source code for the proposed method and experiments online at \url{https://github.com/hsuhr99/measurement_mechanistic_rca}. The repository includes a README, dockerized environment, scripts for generating synthetic data and running experiments, and documentation of the main configuration options. Like the paper, the code is licensed under the CC-BY 4.0 license.

\item {\bf Crowdsourcing and research with human subjects}
    \item[] Question: For crowdsourcing experiments and research with human subjects, does the paper include the full text of instructions given to participants and screenshots, if applicable, as well as details about compensation (if any)? 
    \item[] Answer: \answerNA{} % Replace by \answerYes{}, \answerNo{}, or \answerNA{}.
    \item[] Justification: The contribution does not involve crowdsourcing nor research with human subjects.

\item {\bf Institutional review board (IRB) approvals or equivalent for research with human subjects}
    \item[] Question: Does the paper describe potential risks incurred by study participants, whether such risks were disclosed to the subjects, and whether Institutional Review Board (IRB) approvals (or an equivalent approval/review based on the requirements of your country or institution) were obtained?
    \item[] Answer: \answerNA{} % Replace by \answerYes{}, \answerNo{}, or \answerNA{}.
    \item[] Justification: The paper does not involve crowdsourcing nor research with human subjects.

\item {\bf Declaration of LLM usage}
    \item[] Question: Does the paper describe the usage of LLMs if it is an important, original, or non-standard component of the core methods in this research? Note that if the LLM is used only for writing, editing, or formatting purposes and does \emph{not} impact the core methodology, scientific rigor, or originality of the research, declaration is not required.
    %this research? 
    \item[] Answer: \answerNA{} % Replace by \answerYes{}, \answerNo{}, or \answerNA{}.
    \item[] Justification: The core method development in this research does not involve LLMs as any important, original, or non-standard components.

\end{enumerate}